\documentclass[9.5pt,journal,compsoc]{IEEEtran}

\ifCLASSOPTIONcompsoc
  \usepackage[nocompress]{cite}
\else
  \usepackage{cite}
\fi

\usepackage[ruled]{algorithm2e} 

\usepackage{epsfig}
\usepackage{graphicx}
\usepackage{amssymb,amsmath,amsthm}
\usepackage{algorithmicx}
\usepackage{algpseudocode}
\usepackage{caption}
\usepackage{subcaption}
\usepackage{epstopdf}
\usepackage{dsfont}
\usepackage{float}
\usepackage{nccmath}
\usepackage{tikz}
\usepackage{url}
\usetikzlibrary{positioning,shapes}

\graphicspath{ {figures/} }

\newtheorem{theorem}{Theorem}
\begin{document}
%
\title{Outlier Detection for Robust Multi-dimensional Scaling}

\author{Leonid~Blouvshtein,
        Daniel~Cohen-Or
        }
        
\newcommand{\lb}[1]{{\color{blue}\textbf{ }#1}\normalfont}
\newcommand{\dc}[1]{{\color{magenta}\textbf{ }#1}\normalfont}
\newcommand{\nf}[1]{{\color{red}\textbf{ }#1}\normalfont}

%



\IEEEtitleabstractindextext{%
\begin{abstract}
	Multi-dimensional scaling (MDS) plays a central role in data-exploration, 
dimensionality reduction and visualization. 
State-of-the-art MDS algorithms are not robust to outliers, yielding significant errors in the  embedding even when only a handful of outliers are present.
	In this paper, we introduce a technique to detect and filter outliers based on geometric reasoning. We test the validity of triangles formed by three points, and mark a triangle as broken if its triangle inequality does not hold. 
	The premise of our work is that unlike inliers, outlier distances tend to break many triangles. 
	Our method is tested and its performance is evaluated on various datasets and distributions of outliers. We demonstrate that for a reasonable amount of outliers, e.g., under $20\%$, our method is effective, and leads to a high embedding quality.
	
\end{abstract}
}

\maketitle

\section{Introduction}
\IEEEPARstart{M}ulti-dimensional Scaling (MDS) is a fundamental problem in data analysis and information visualization. MDS takes as input a 
distance matrix $D$, containing all $\binom{N}{2}$ pairs of distances between elements $x_i$, and embeds the elements in 
$d$ dimensional space such that the pairwise distances $D_{ij}$ are preserved as much as possible by $|| x_i - x_j ||$ in the embedded space.
When the distance data is outlier-free, state-of-the-art methods (e.g., SMACOF) provide satisfactory solutions 
\cite{de1988convergence,leeuw2008multidimensional}. These solutions are based on an optimization of the so-called \textbf{stress function},
which is a sum of squared errors between the dissimilarities $D_{ij}$ and their corresponding embedding inter-vector distances:
\begin{equation*}\label{eq:stress}
Stress_D(x_1,x_2,...x_N)  = \sum_{i \neq j}^{} (D_{ij} - || x_i - x_j ||)^2.
\end{equation*}

In many real-world scenarios, input distances may be noisy or contain outliers, due to
malicious acts, system faults, or erroneous measures.
Many MDS techniques deal with noisy data \cite{de2004sparse,chan2009efficient}, but little attention has been given to outliers 
\cite{spence1989robust,forero2012sparsity,cayton2006robust}. 
We refer to outliers as opposed to noise, as distances that are significantly different than their corresponding true values.

\begin{figure}[t]
	\begin{center}
		\setlength\fboxsep{0pt}
		\setlength\fboxrule{0pt}
		\fbox{\rule{0pt}{0in}
			\includegraphics[trim={1cm 0 2cm 0},clip,width=0.5\linewidth]{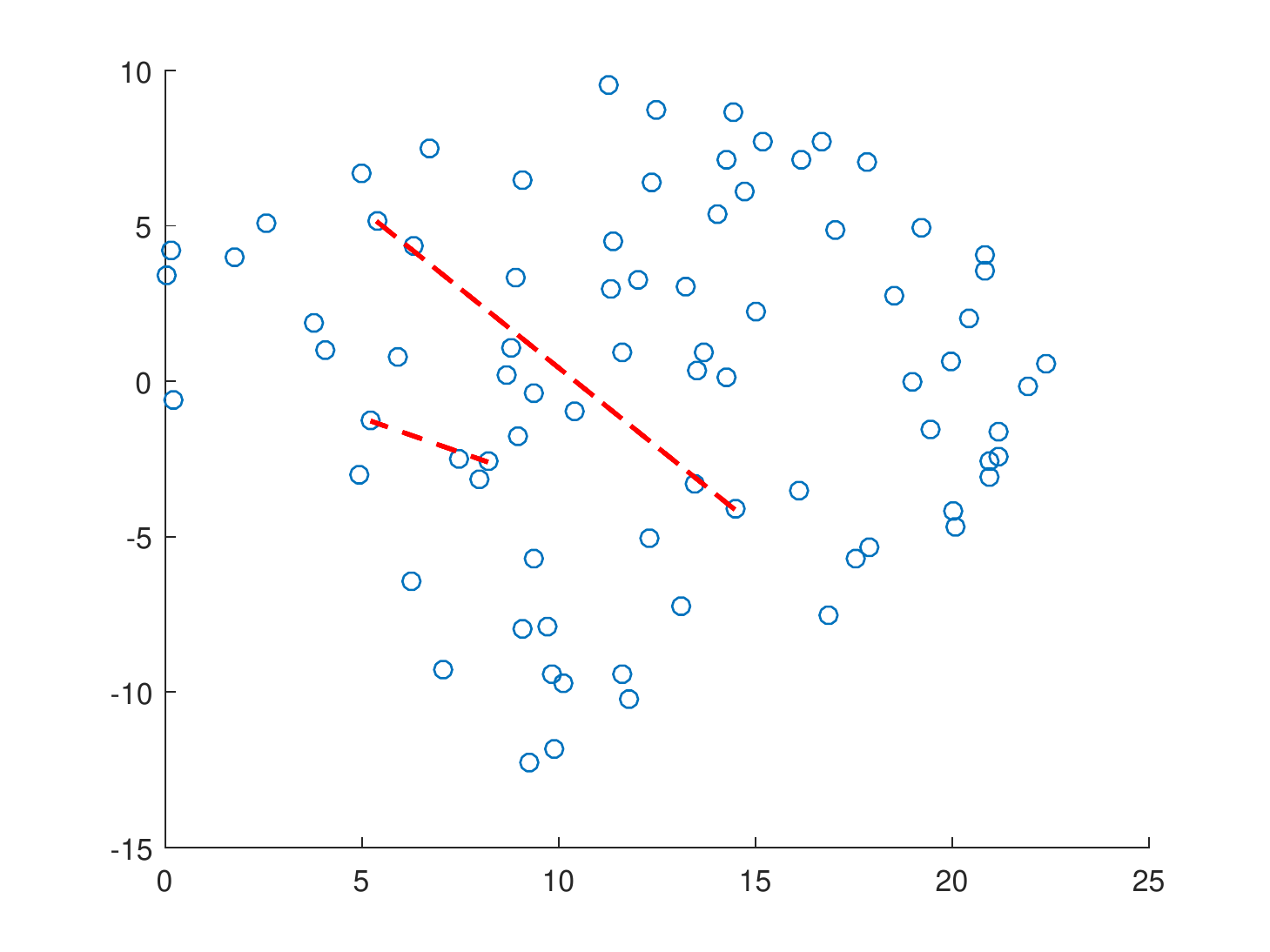} \includegraphics[trim={1cm 0 2cm 0},clip,width=0.5\linewidth]{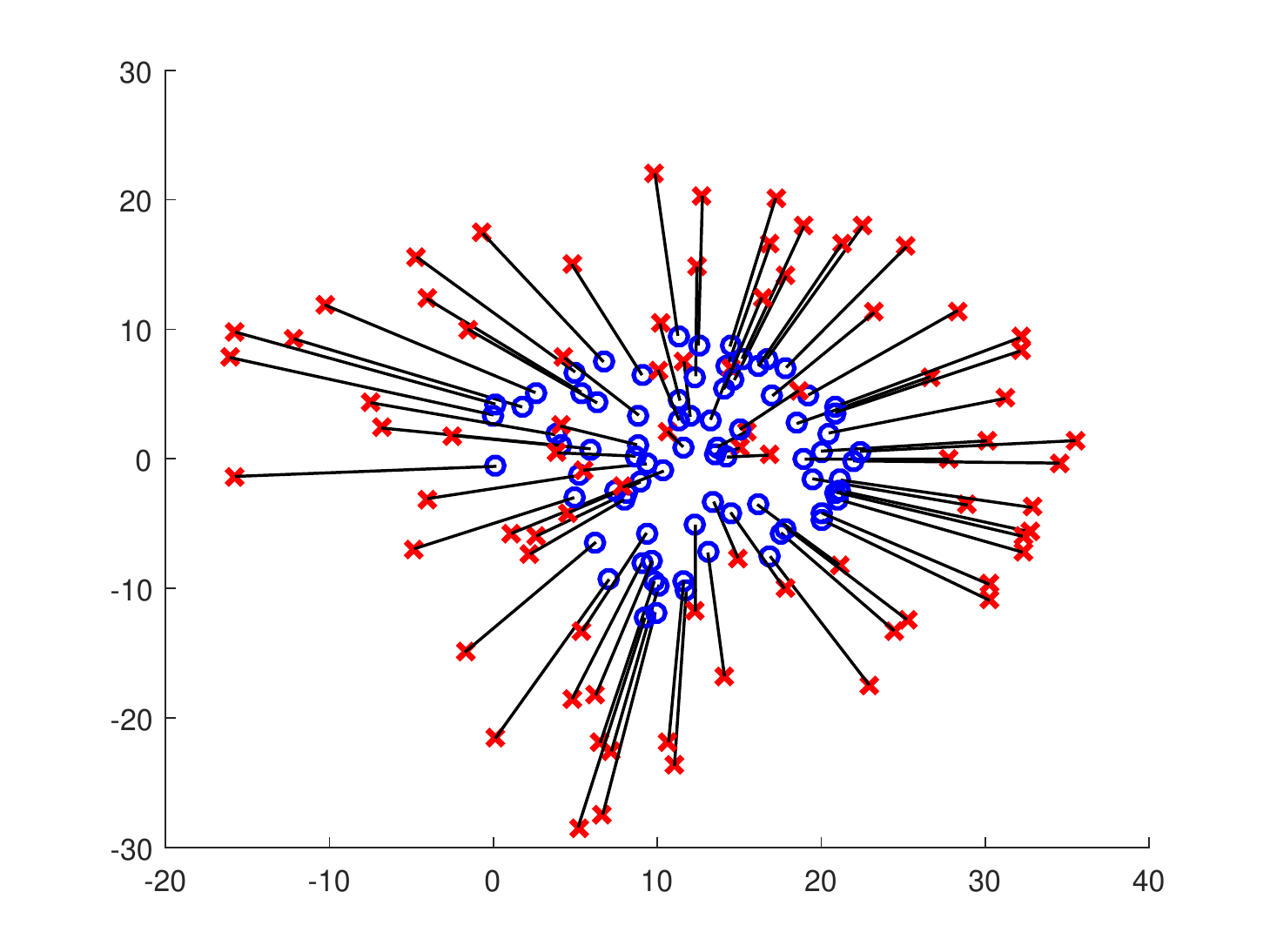}
		}		
	\end{center}
	\caption{Two outlier distances (marked in dashed lines on the left) lead to a significant distortion in the embedding,
       as reflected by the large offsets between ground-truth and embedded positions, shown on the right.
        }
	\label{fig:1}
\end{figure}

Developing a robust MDS is challenging since even a small portion of outliers can lead to significant errors.  This sensitivity of 
MDS to outliers is demonstrated in Figure \ref{fig:1}, where only two pairwise distances (out of 435 pairs) are erroneous (colored in red in (a)), cause a strong distortion in the embedding (b).
To highlight the embedding errors, we draw lines connecting between the ground truth positions and the embedded positions.

In this paper, we introduce a robust MDS technique, which detects and removes outliers from the data and hence provides a better embedding. 



Our approach is based on geometric reasoning with which the outliers can be identified by analyzing the given input distances. This approach follows the well-known idiom "prevention is better than a cure". 
That is, instead of recovering from the damage that the outliers cause to the MDS optimization, we prevent them in the first place, 
by detecting and filtering them out.

We treat the distances as a complete graph of $\binom{N}{2}$ edges. Each edge is associated with its corresponding distance and 
forms $N-2$ triangles with the rest of the $N-2$ nodes.
The premise of our work is that an outlier distance tends to \textbf{break} many triangles. We refer to a triangle as \textbf{broken} 
if its triangle inequality does not hold. As we shall show, while inlier edges participate in a rather small number of broken triangles, 
outlier edges participate in many. This allows us to set a conservative threshold and classify the edges and their associated distances 
as inliers and outliers (See Figure \ref{fig:bad_edge_visualization}).

Generally speaking, MDS is an overdetermined problem since the distance matrix contains many more distances than necessary to solve the 
problem correctly.
Hence, the idea is to detect distances that are suspected to be outliers and remove them before applying the MDS. In the following, we denote and refer to our robust MDS method with TMDS.
As we shall show, our technique succeeds in detecting and removing most of the outliers without any parameters, while incurring a rather small number of false positives, to facilitate a more accurate embedding. We tested, analyzed and evaluated our method on a large number 
of datasets with varying portions of outliers from various distributions.

\begin{figure*}[h]
	\centering
	\begin{subfigure}{0.21\textheight}
		\includegraphics[trim={1cm 0cm 2cm 0},clip,width=\linewidth]{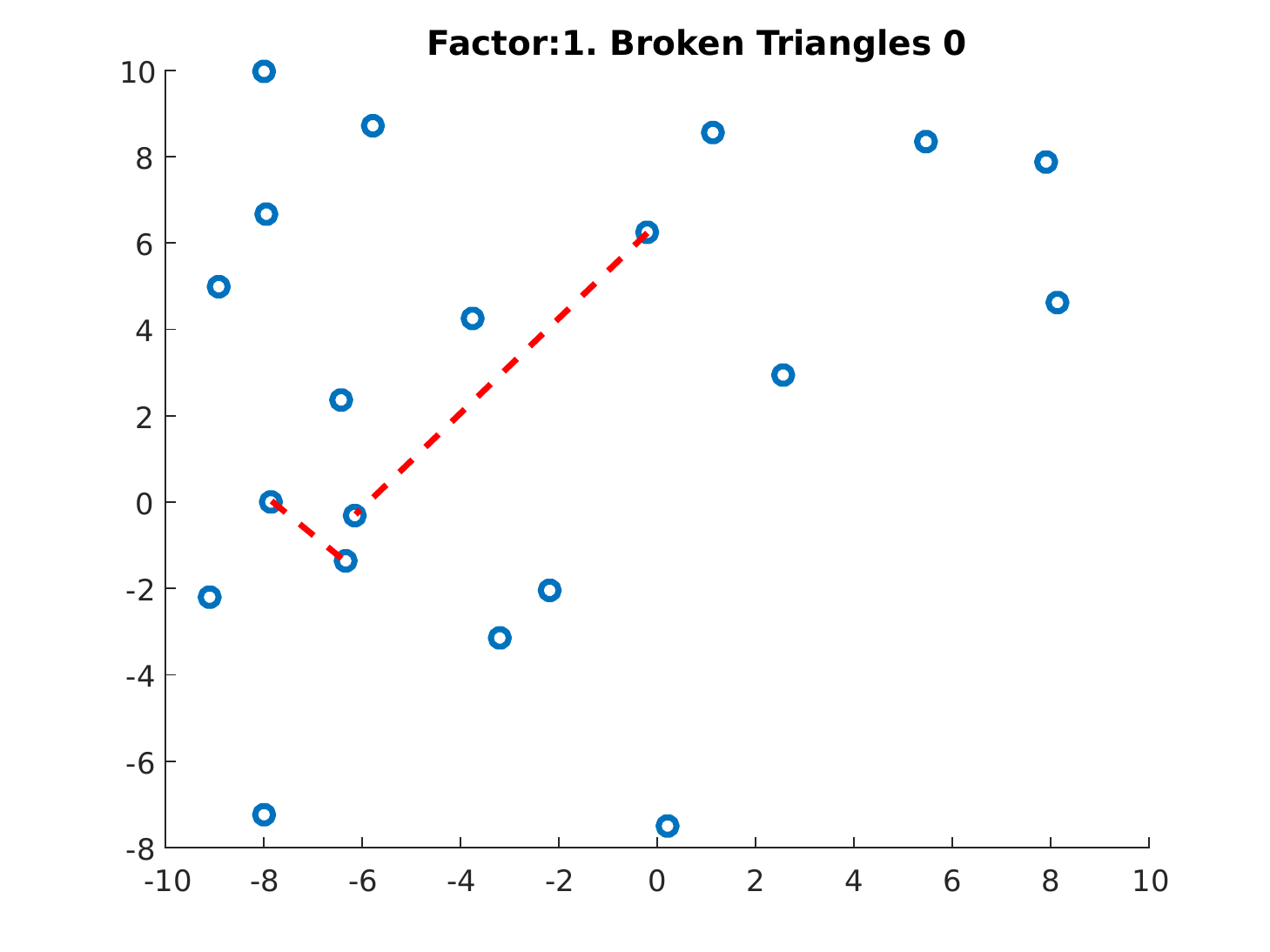}
		\caption{}
	\end{subfigure}
	\begin{subfigure}{0.21\textheight}
		\includegraphics[trim={1cm 0cm 2cm 0},clip,width=\linewidth]{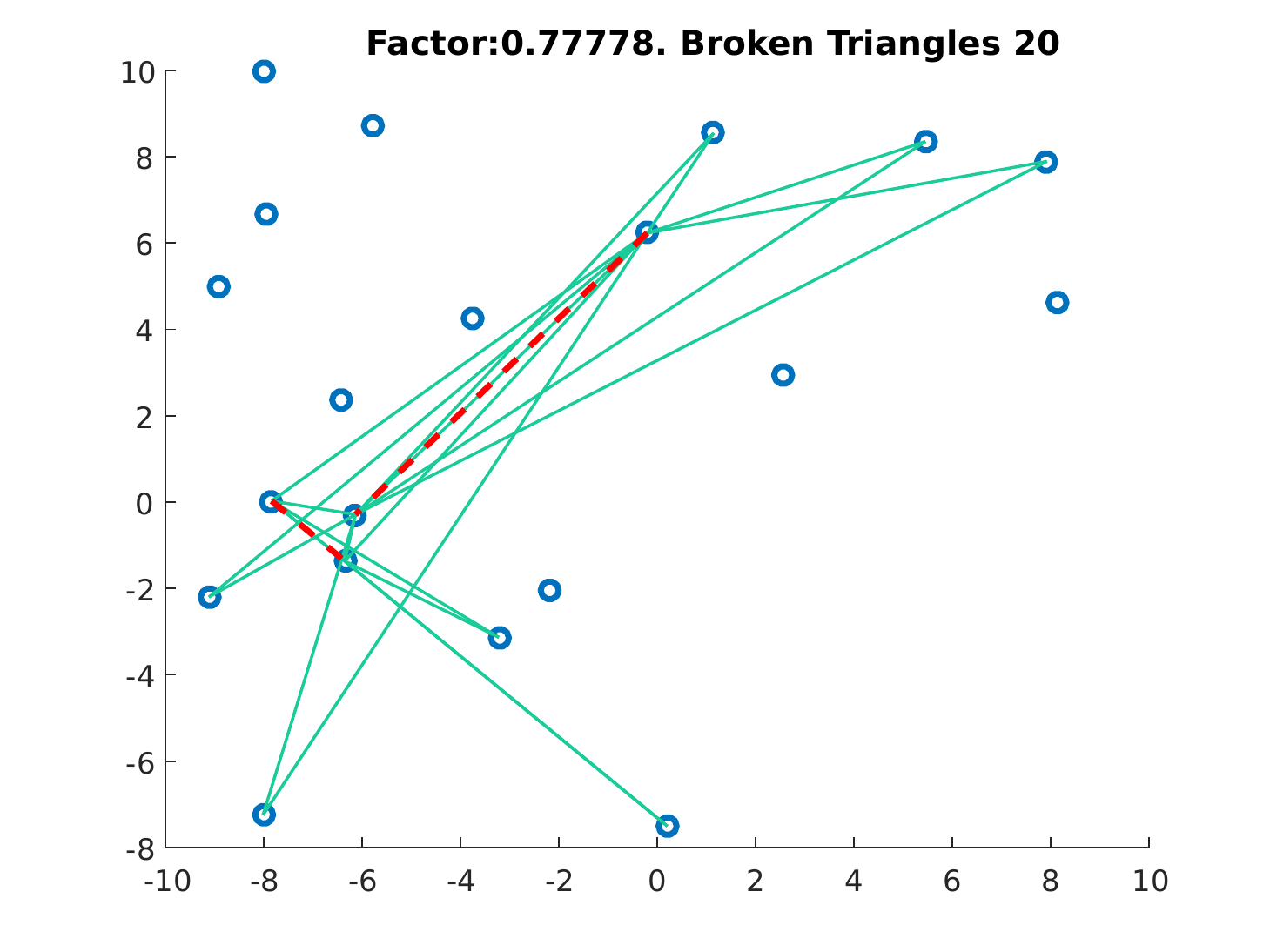}
		\caption{}
	\end{subfigure}
	\begin{subfigure}{0.21\textheight}
		\includegraphics[trim={1cm 0cm 2cm 0},clip,width=\linewidth]{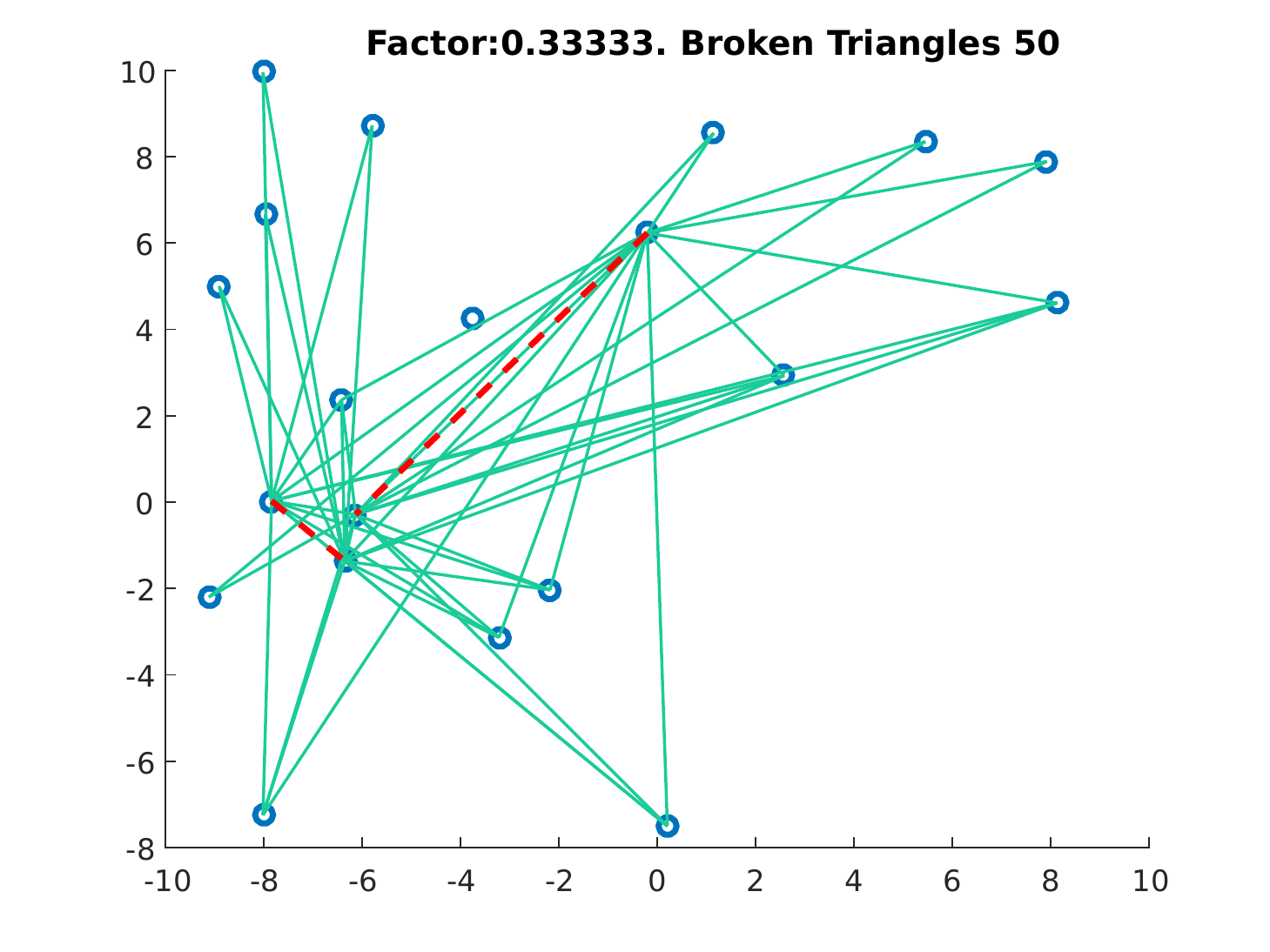}
		\caption{}
	\end{subfigure}

	\caption{The effect of distorting two distances, marked by red dashed lines (a). The green lines represent edges that violate 
        the triangle inequality (b). A stronger distortion leads to a larger number of broken triangles (c)}
	\label{fig:bad_edge_visualization}
\end{figure*}

\begin{figure}[b]
	\centering
	\begin{subfigure}[b]{0.11\textheight}
		\includegraphics[trim={1cm 0 2cm 0},clip,width=\linewidth]{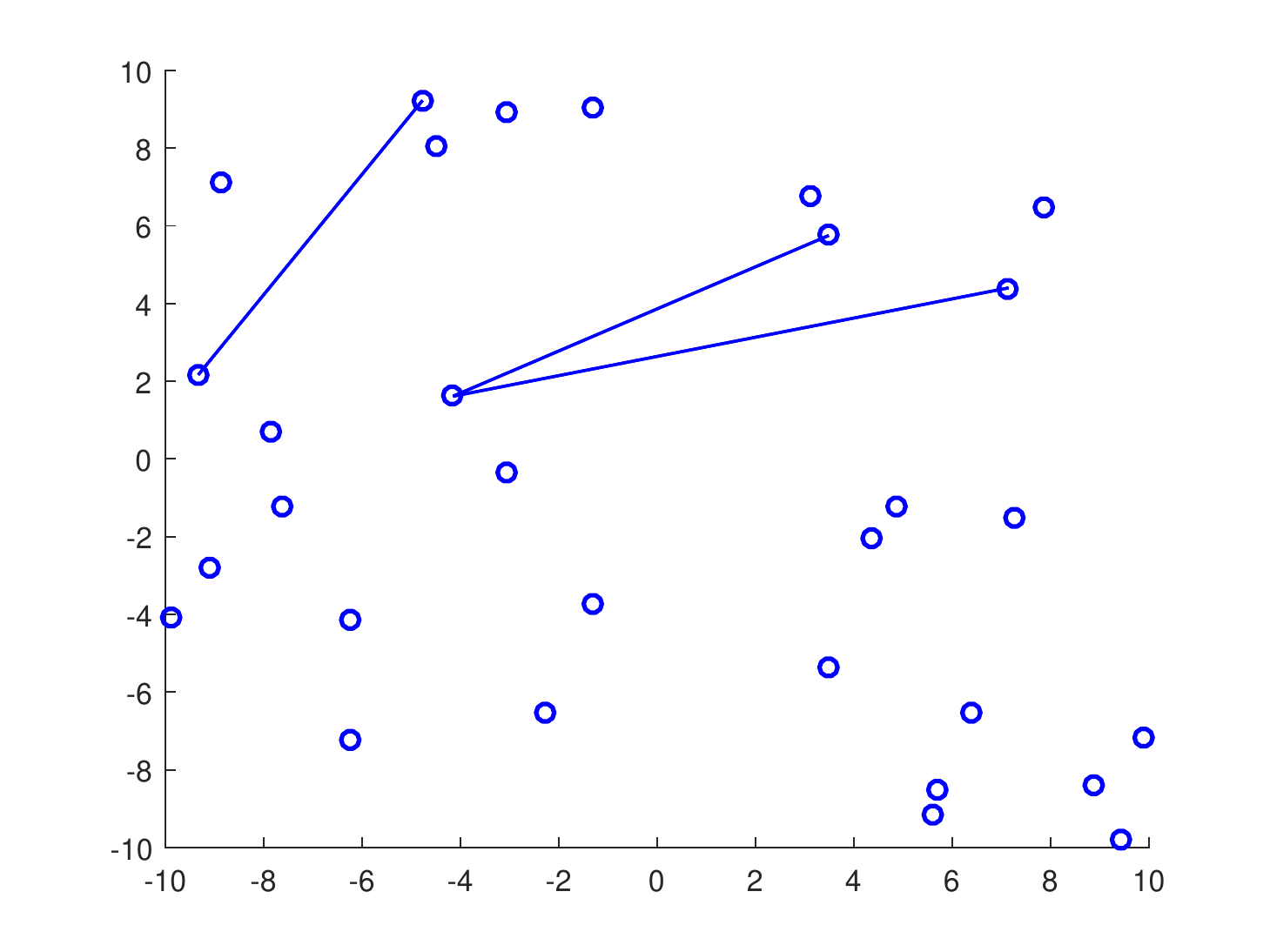}
		\caption{}
	\end{subfigure}
	\begin{subfigure}[b]{0.11\textheight}
		\includegraphics[trim={1cm 0 2cm 0},clip,width=\linewidth]{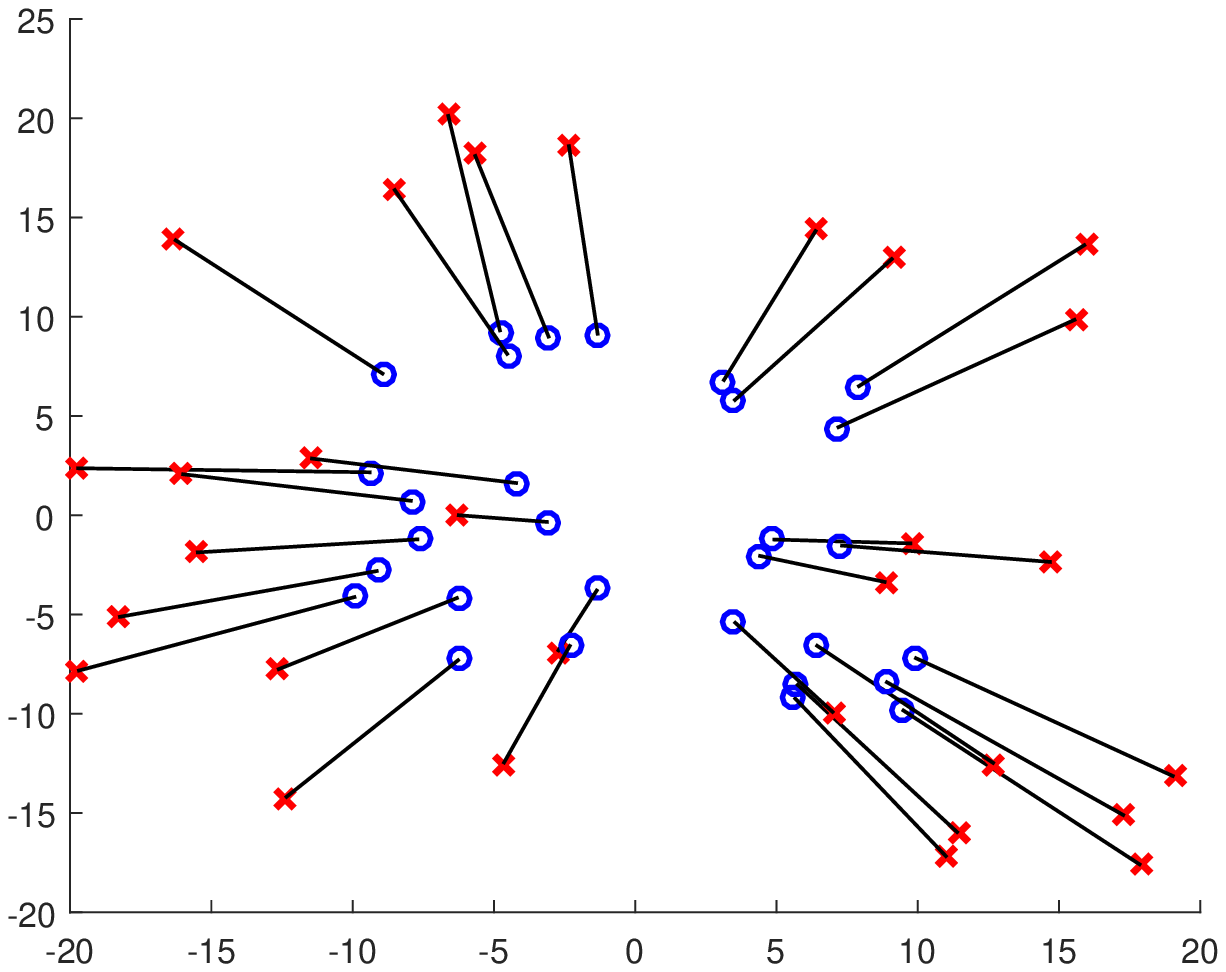}
		\caption{}
	\end{subfigure}
	\begin{subfigure}[b]{0.11\textheight}
		\includegraphics[trim={1cm 0 2cm 0},clip,width=\linewidth]{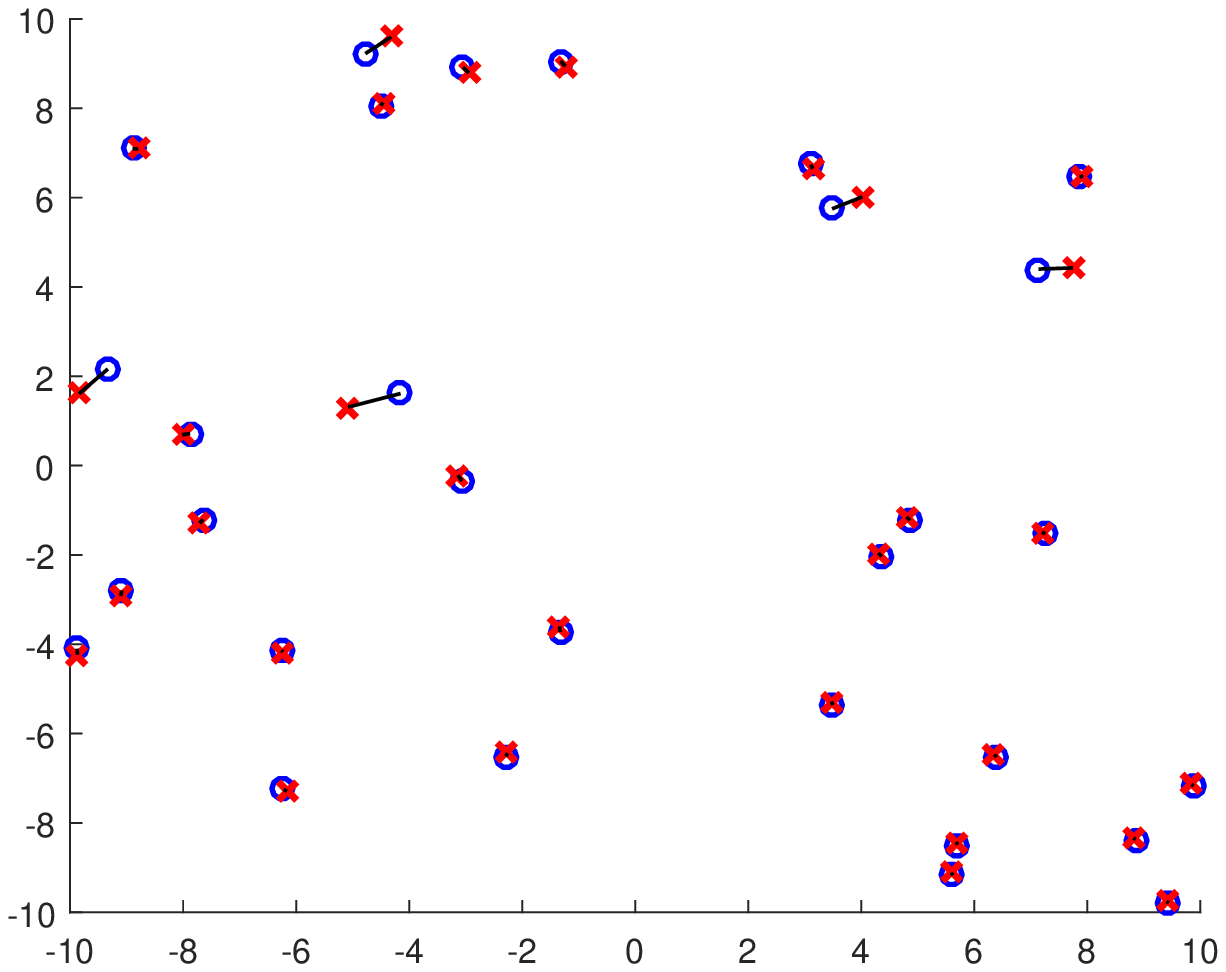}
		\caption{}
	\end{subfigure}
	
	\caption{(a) The blue lines represent the enlarged distances. (b) embedding with SMACOF. (c) Sammon embedding.}
	\label{fig:sammon_large}
\end{figure}

\begin{figure}[b]
	\centering
	\begin{subfigure}[b]{0.11\textheight}
		\includegraphics[trim={1cm 0 2cm 0},clip,width=\linewidth]{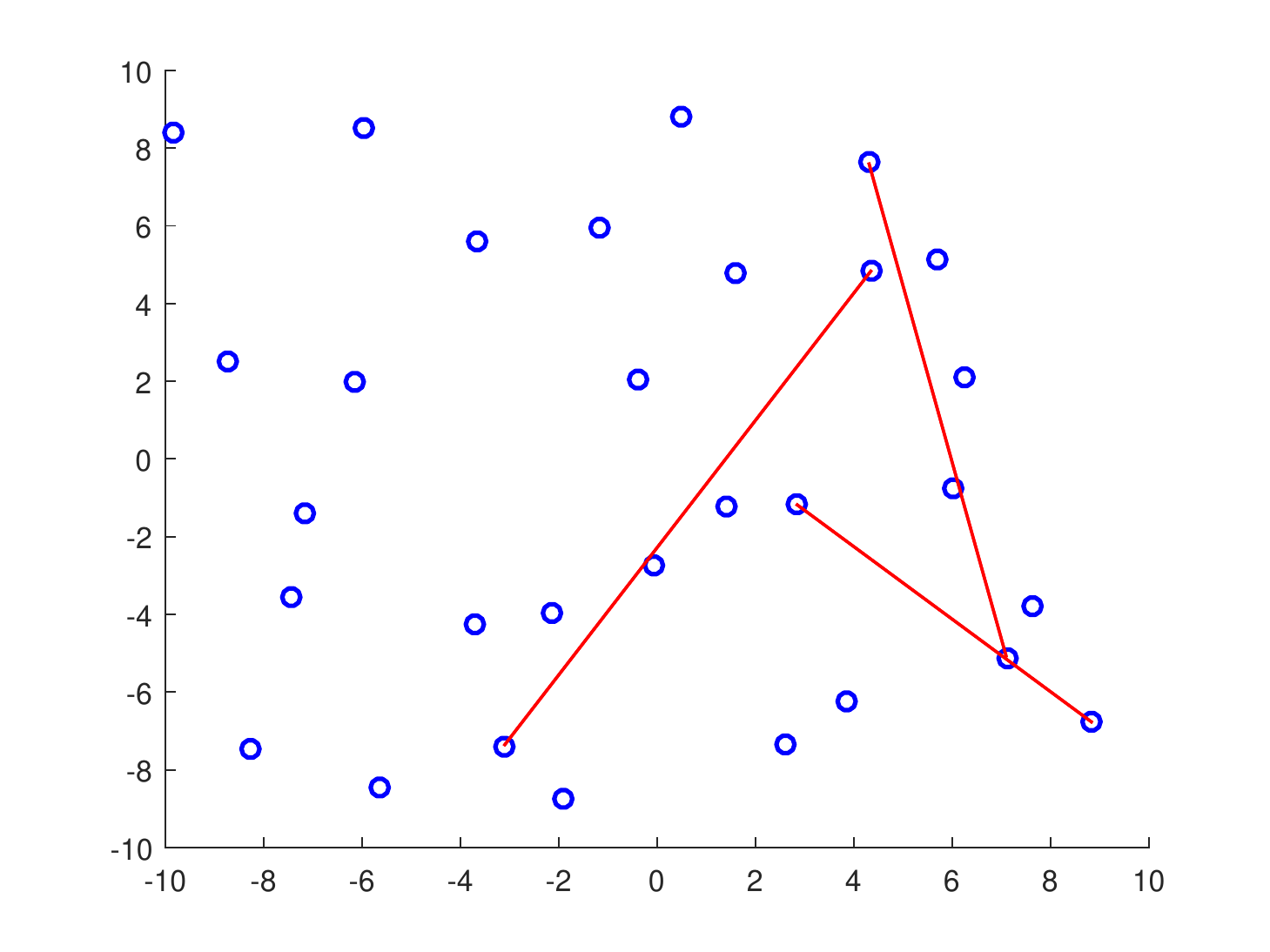}
		\caption{}
	\end{subfigure}
	\begin{subfigure}[b]{0.11\textheight}
		\includegraphics[trim={1cm 0 2cm 0},clip,width=\linewidth]{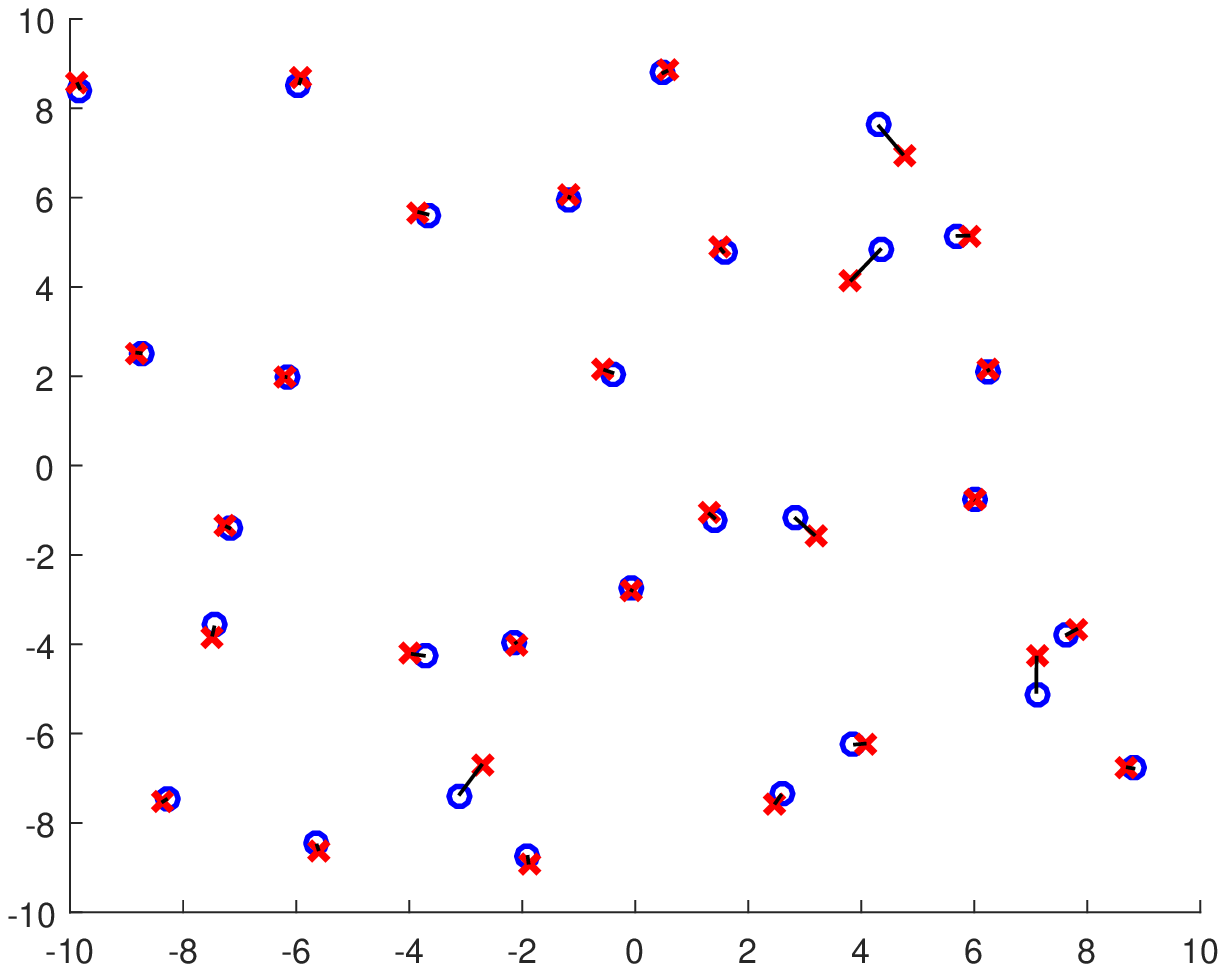}
		\caption{}
	\end{subfigure}
	\begin{subfigure}[b]{0.11\textheight}
		\includegraphics[trim={1cm 0 2cm 0},clip,width=\linewidth]{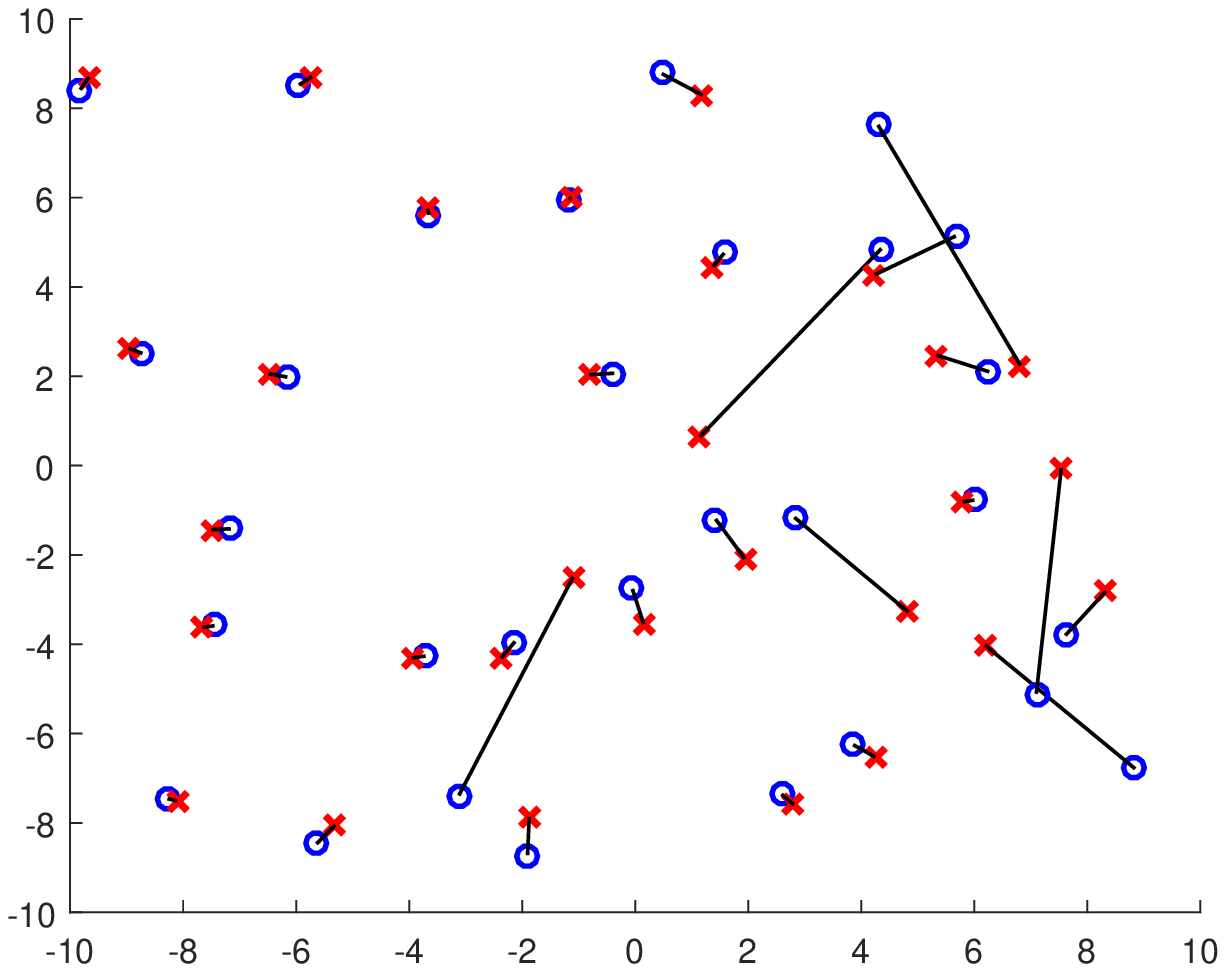}
		\caption{}
	\end{subfigure}

	\caption{(a) The blue lines represent the shortened distances. (b) embedding with SMACOF. (c) Sammon embedding.}
	\label{fig:sammon_short}
\end{figure}

\section{Background}

MDS was originally used and developed in the field of Psychology, as a means to visualize perceptual relations among 
objects\cite{Krusk64_0,Shep62_0}.
Nowadays, MDS is used in a wide variety of fields, such as marketing \cite{Market_0}, and graph embedding \cite{shepard1980multidimensional}.
Most notably, MDS plays a central role in data exploration \cite{borg2005modern, buja2008data,buja2002visualization}, and computer graphics applications like texture mapping \cite{zigelman2002texture}, shape classification and retrieval \cite{chen20103d,pickup2014shrec,Li:2015} and more.



Several methods were suggested to handle outliers in the data (e.g., \cite{spence1989robust,cayton2006robust}).
Using Sammon weighting \cite{sammon1969nonlinear} leads to the following stress function:
\begin{equation*}\label{eq:sammon_stress}
\sum_{i \neq j}^{} \dfrac{(D_{ij} - || x_i - x_j ||)^2} {D_{ij}}.
\end{equation*}

This objective function can effectively be considered as robust to elongated distances since it decreases
the weights of long distances.
We differentiate between two types of outliers: \textbf{larger} and \textbf{shorter} outliers (colored in blue and red, respectively, 
in Figures \ref{fig:sammon_large} and \ref{fig:sammon_short}), since their characteristics and effects are different and may thus require 
different treatments. In (a) we show 2D data elements, where the outliers are marked in red and blue. In (b) we show their positions recovered by applying a state-of-the-art MDS (i.e., SMACOF) and in (c) the results of Sammon
method \cite{sammon1969nonlinear}.
As can be observed, Sammon method can deal well with elongated distances, by assigning them with low weights. However,
shortened distances are not dealt with well, as they are assigned larger weights which lead to a distorted embedding.

The most related work to ours is the method presented by Forero and Giannakis \cite{forero2012sparsity}, hereafter referred to as FG12.
They use an objective function $F(X,O)$ that aims to find an embedding $X$ and an outliers matrix $O$ that minimize the following:
\begin{equation*}\label{eq:sammon_stress}
	\sum_{i < j}^{} (D_{ij} - || x_i - x_j || - O_{ij})^2 + \lambda \sum_{i < j}^{} \mathds{1} (O_{ij} \neq 0),
\end{equation*}
where $\lambda$  regulates the number of non-zero values in $O$ that represent outliers.  


Setting the size of $\lambda$ to control the sparsity of $O_{ij}$ is not easy. If $\lambda$ is too big, too few outliers are detected; 
if it is too small, too many edges are treated as outliers. As we shall show, close values of $\lambda$ can lead to different results. 
This phenomenon is shown in Figure \ref{fig:lambda_experiments} (a-c). Thus, careful tuning of $\lambda$ is required to achieve good
results. This is an overly complex process, since, as well shall show, the algorithm is also sensitive to the initial guess.


Note that $X$ has $d \times N$ unknown variables and $O$ has $\binom{N}{2}$ variables. 
This amounts to a considerable increase in the number of parameters, and hence it is significantly harder to optimize FG12 compared to
SMACOF.
This is evident in Figure ~\ref{fig:lambda_experiments} where we show (d-f) that with the \textbf{same} $\lambda$ applied to the
\textbf{same} dataset, but with different initial guesses. As shown, the three initial guesses yield a different number of edges
that are considered as outliers. This behavior of the FG12 method can also be observed in Figure \ref{fig:initial_guess}. Note that for
the yellow curve, $\lambda = 1.8$ is the value that detects the correct number of outliers. However, for the same value of $\lambda$ 
the blue plot detects most of the edges as outliers. This highlights the sensitivity of the FG12 method and emphasizes that we cannot set
the value of $\lambda$ even when we have a good estimation of the number of outliers in the system.

\begin{figure}[t]
	\centering
	\begin{subfigure}[b]{0.11\textheight}
		\includegraphics[trim={1cm 0 2cm 0},clip,width=\linewidth]{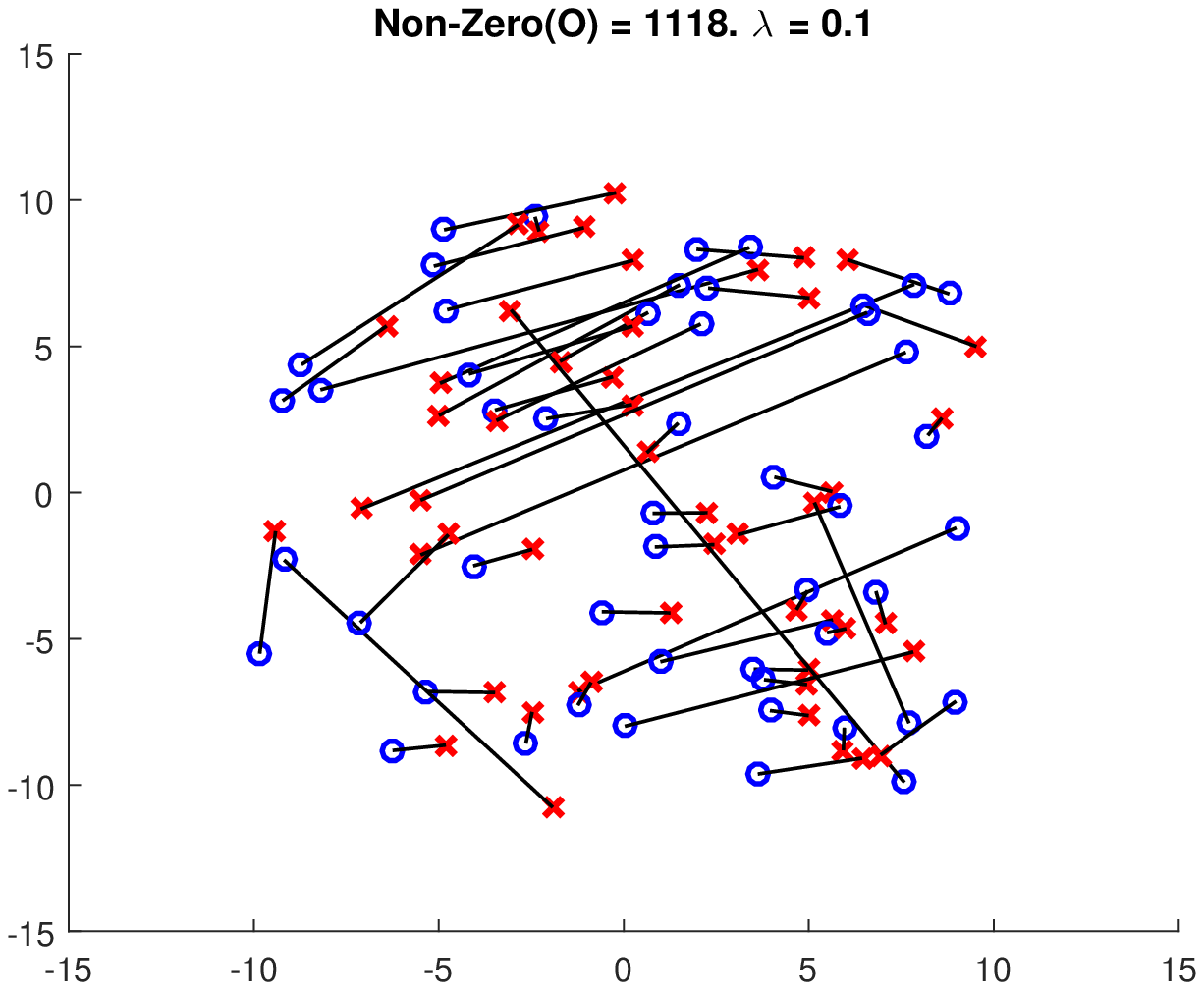}
		\caption{}
	\end{subfigure} %
    \begin{subfigure}[b]{0.11\textheight}
       	\includegraphics[trim={1cm 0 2cm 0},clip,width=\linewidth]{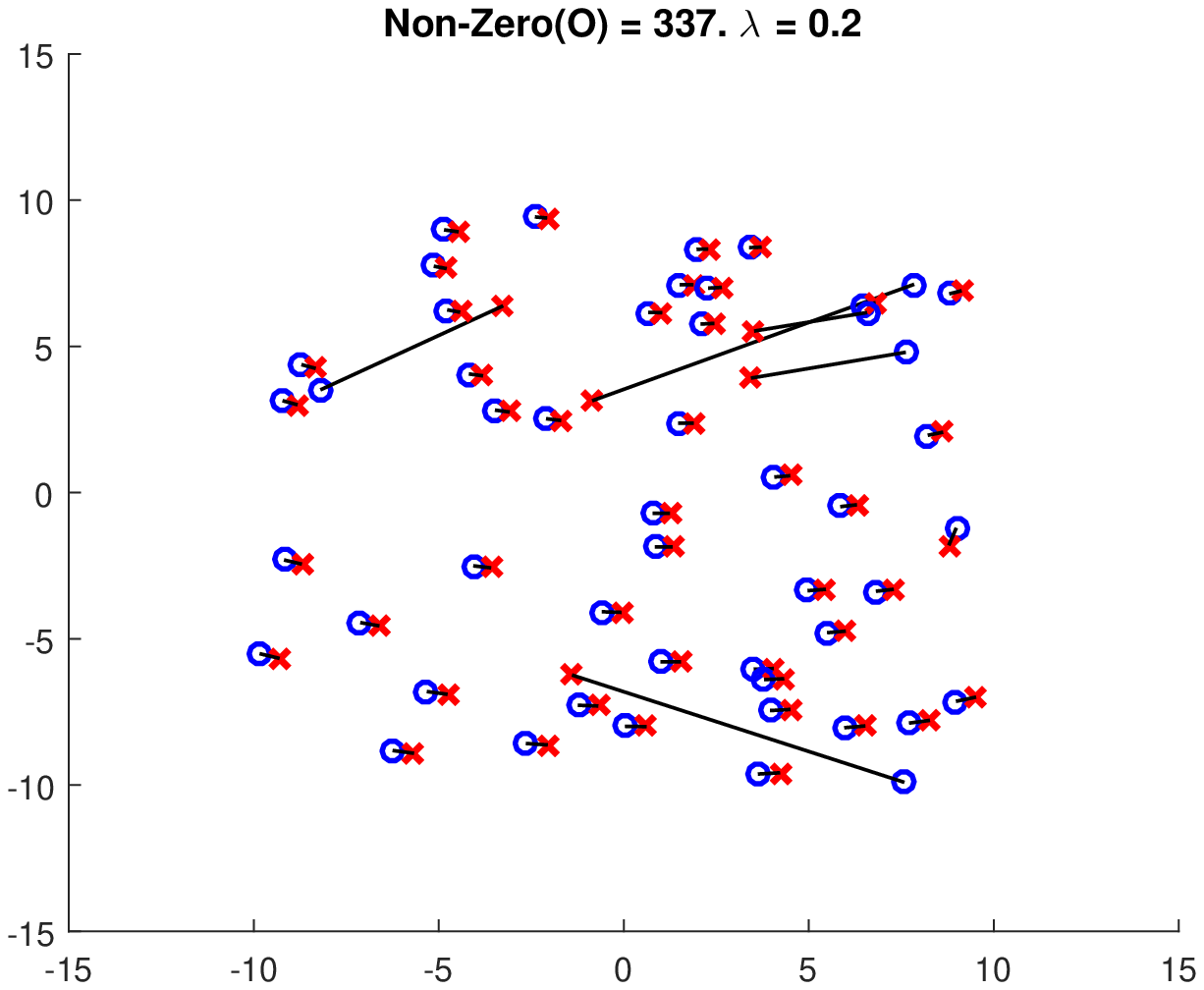}
       	\caption{}
    \end{subfigure} %
    \begin{subfigure}[b]{0.11\textheight}
	 	\includegraphics[trim={1cm 0 2cm 0},clip,width=\linewidth]{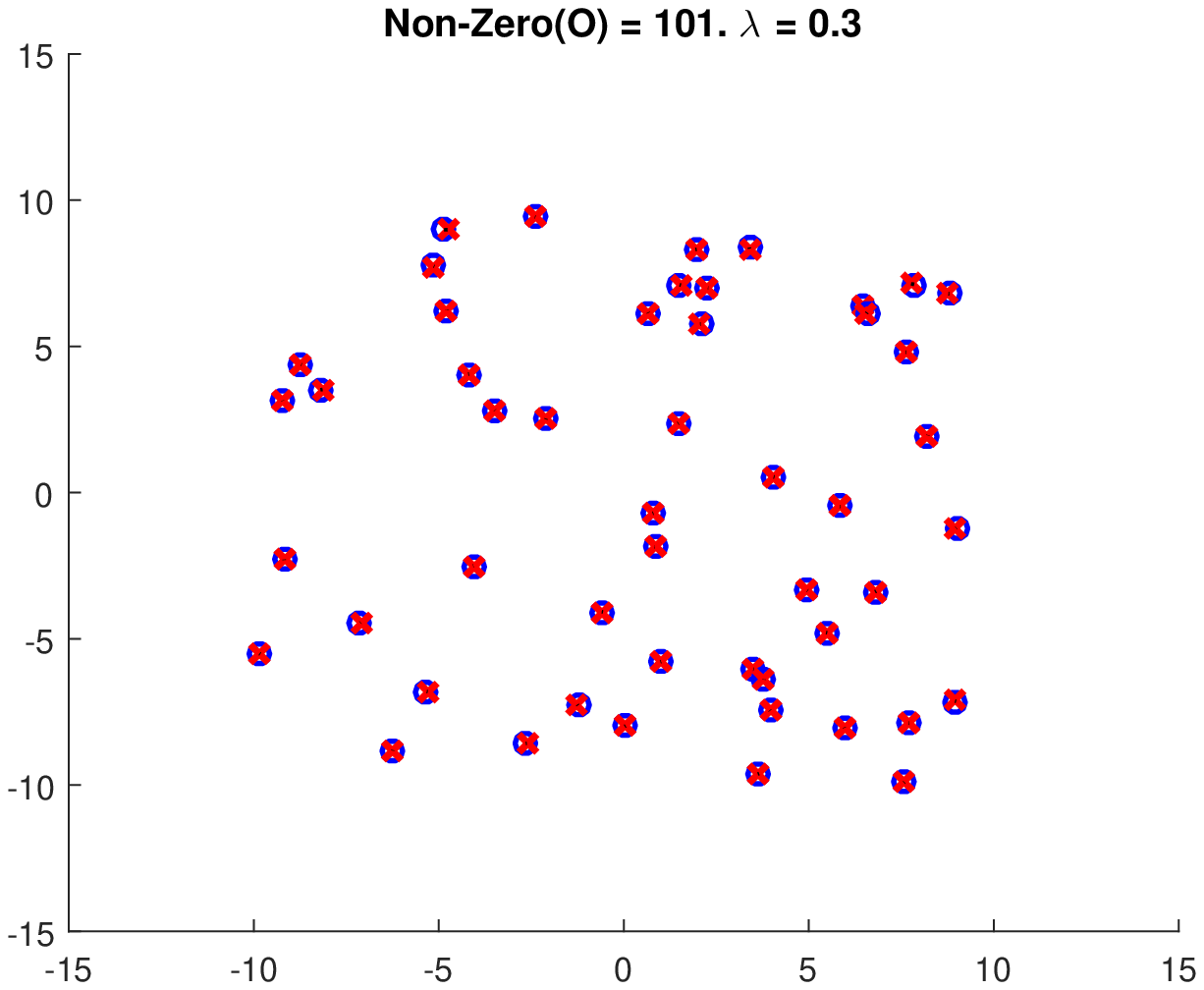} 
	 	\caption{}
	\end{subfigure}
	\begin{subfigure}[b]{0.11\textheight}
		\includegraphics[trim={1cm 0 2cm 0},clip,width=\linewidth]{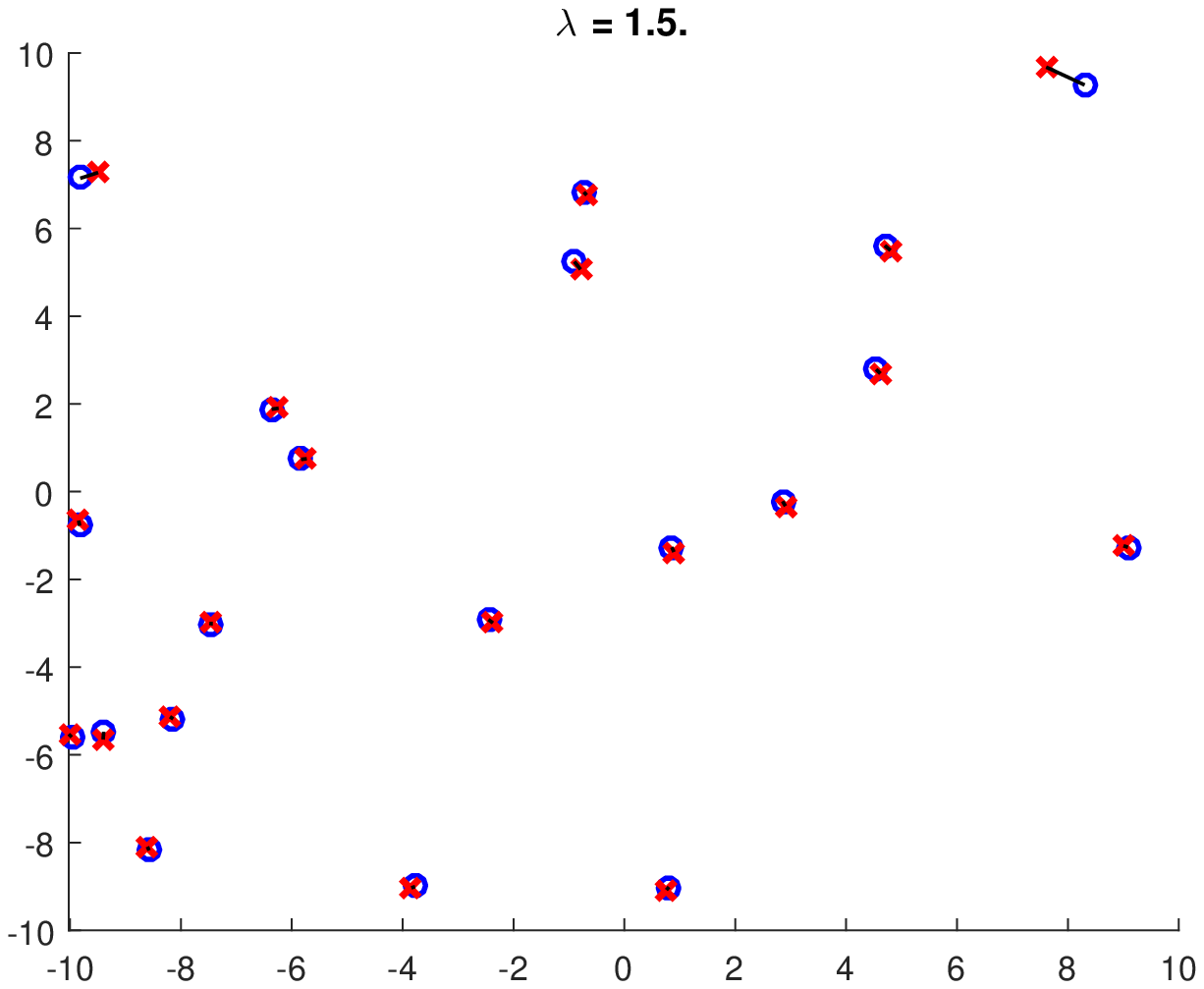}
		\caption{}
	\end{subfigure}
	\begin{subfigure}[b]{0.11\textheight}
		\includegraphics[trim={1cm 0 2cm 0},clip,width=\linewidth]{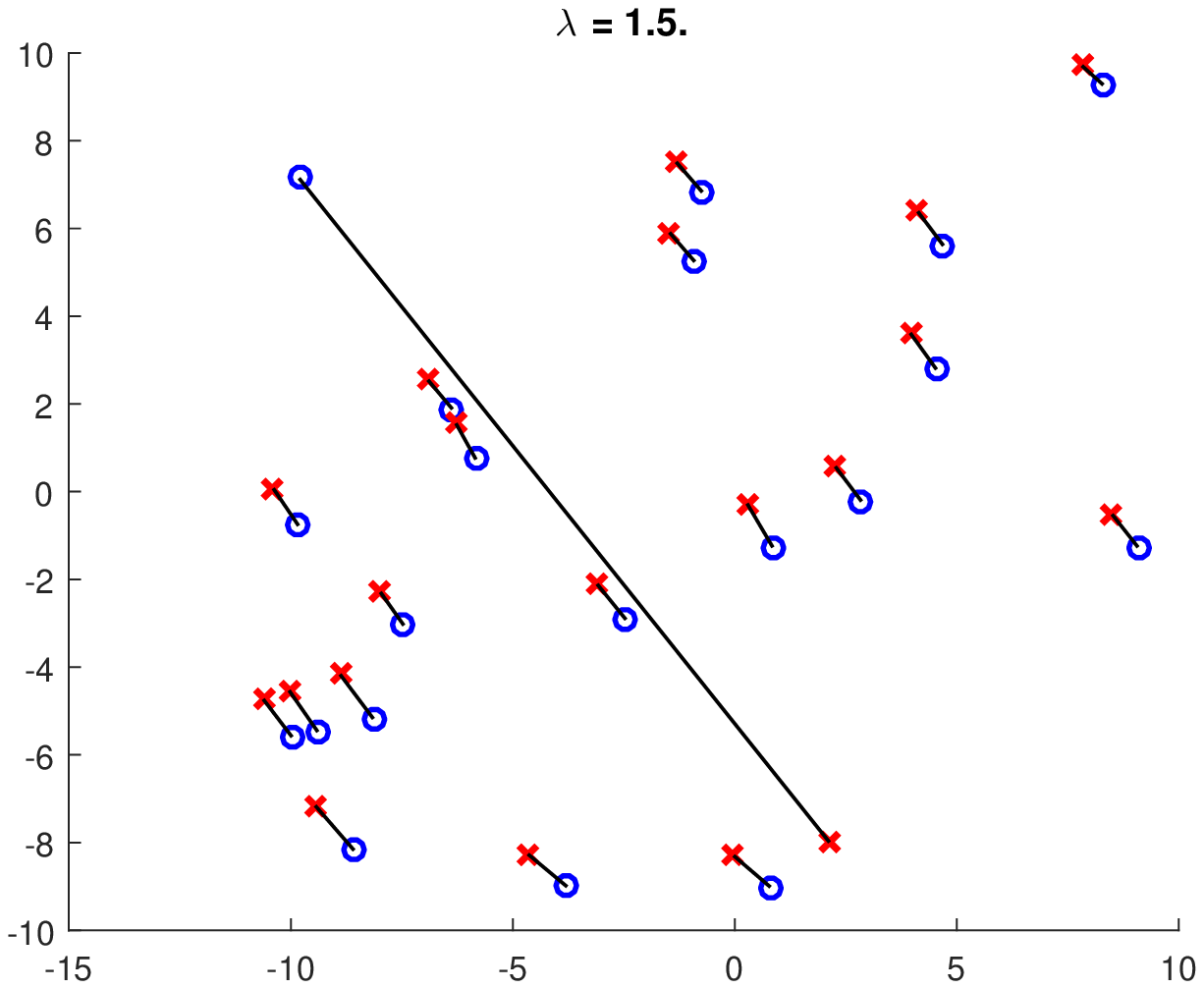}
		\caption{}
	\end{subfigure}
	\begin{subfigure}[b]{0.11\textheight}
		\includegraphics[trim={1cm 0 2cm 0},clip,width=\linewidth]{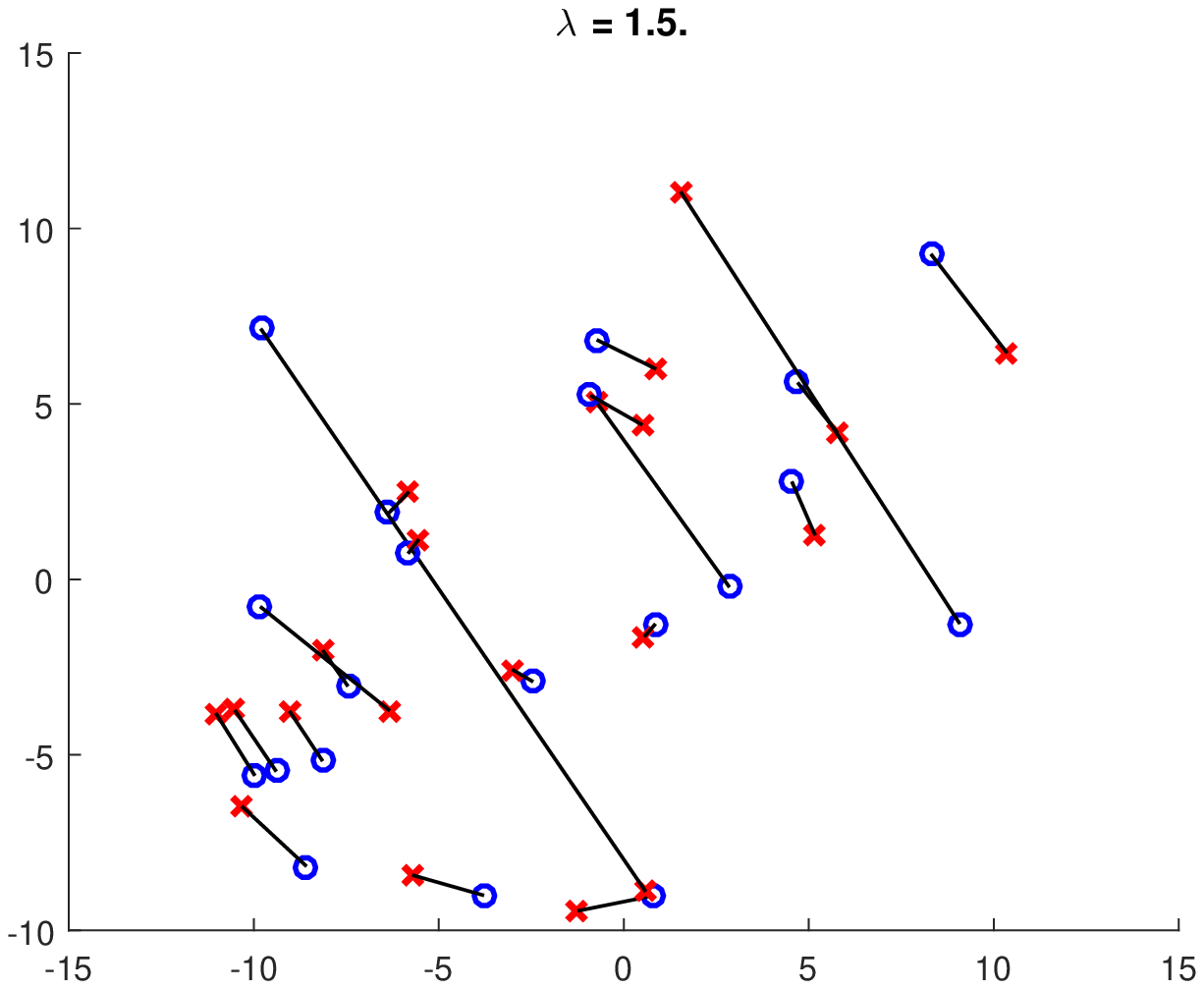}
		\caption{}
	\end{subfigure}

	\caption{(a-c)  Different $\lambda$ applied to the same dataset with the same initial guess, leads to different embedding 
        qualities. (d-f) Same $\lambda$ applied to the same datasets with different initial guesses, yields different embedding qualities.}

	\label{fig:lambda_experiments}
\end{figure}

\begin{figure}[b]
	\begin{center}
		\setlength\fboxsep{0pt}
		\setlength\fboxrule{0pt}
		\fbox{\rule{0pt}{0in}
			\includegraphics[width=\linewidth]{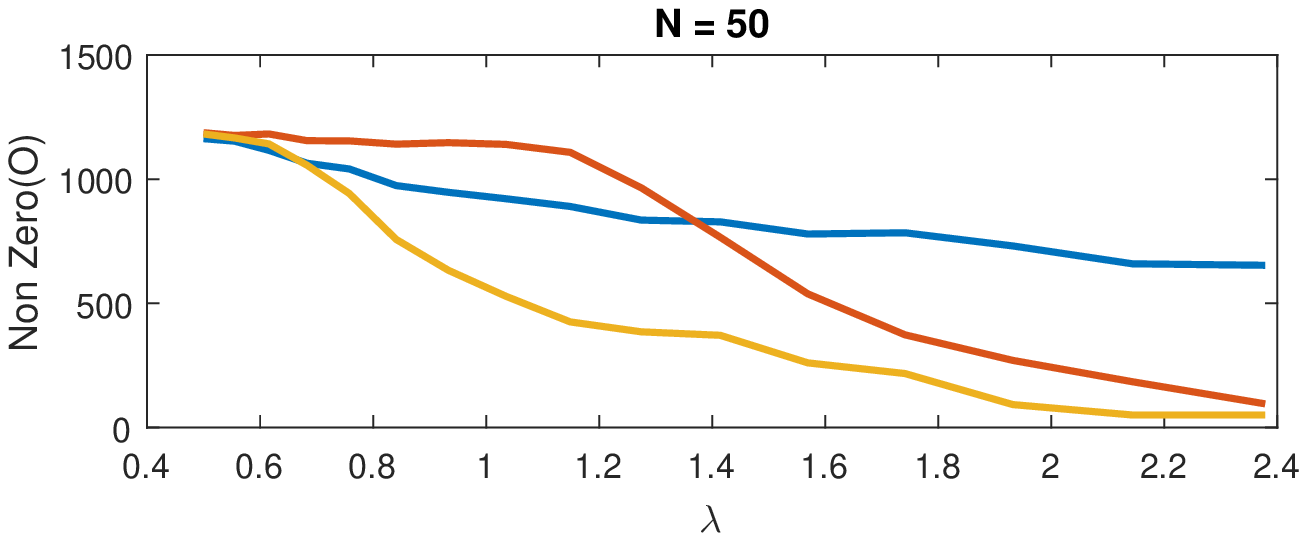}
			
		}		
	\end{center}
	\caption{This graph presents the number of non-zero elements in O (which represent outliers) as a function of $\lambda$. 
    The three plots were generated using different initial guesses that were uniformly sampled. This suggests that the FG12 method is overly sensitive to the initial guess.
	For this experiment we used $N = 50$ ($1225$ edges) and $100$ outliers}
	\label{fig:initial_guess}
\end{figure}

\section{Detecting outliers}

Our technique estimates the likelihood of each distance to be an outlier.
We treat the $\binom{N}{2}$ distances as a complete graph of $\binom{N}{2}$ edges connecting the $N$ vertices.
Each edge is associated with its corresponding distance and forms $N-2$ triangles with the rest of the $N-2$ elements. The key idea is that inlier edges participate in a rather small number of broken triangles, while outlier edges participate in many. As we shall see, by analyzing 
the histogram of broken triangles, we can set a conservative threshold and classify the edges and their associated distances as 
inliers and outliers. See Figure \ref{fig:bad_edge_visualization}.

Let $D$ represent the pairwise distances among graph vertices. In the presence of outliers, some of the edges do not represent a
correct Euclidean relation. In particular, an erroneous edge length tends to break a triangle formed by the edge and its two endpoint 
vertices, and a third vertex from among the rest of the $N-2$ vertices.
Recall that here, a broken triangle is one for which the triangle inequality does not apply.

We can easily identify all the broken triangles by traversing all triangles in the graph. For any triangle with edges of length
$(d_1, d_2, d_3)$ where $d_1 \leq d_2 \leq d_3$, we test: $ d_1 + d_2 < d_3 $ 
We then count for each of the edges in the graph, the
number of broken triangles it participates in.
This yields a histogram $H$, where $H(b)$ counts the number of edges that participate
in $b$ broken triangles. Figure \ref{fig:histogram} depicts such a typical histogram. As can be observed, most of the edges participate
in a small number of broken triangles. The long tail of the histogram is associated with outliers.

It should be noted that an outlier edge does not necessarily break all its triangles, but in large numbers it stands out, and as we 
shall see, is likely to be detected.

 \begin{figure}[t]
 	\begin{center}
 		\setlength\fboxsep{0pt}
 		\setlength\fboxrule{0pt}
 		\fbox{\rule{0pt}{0in}
 			\includegraphics[width=\linewidth]{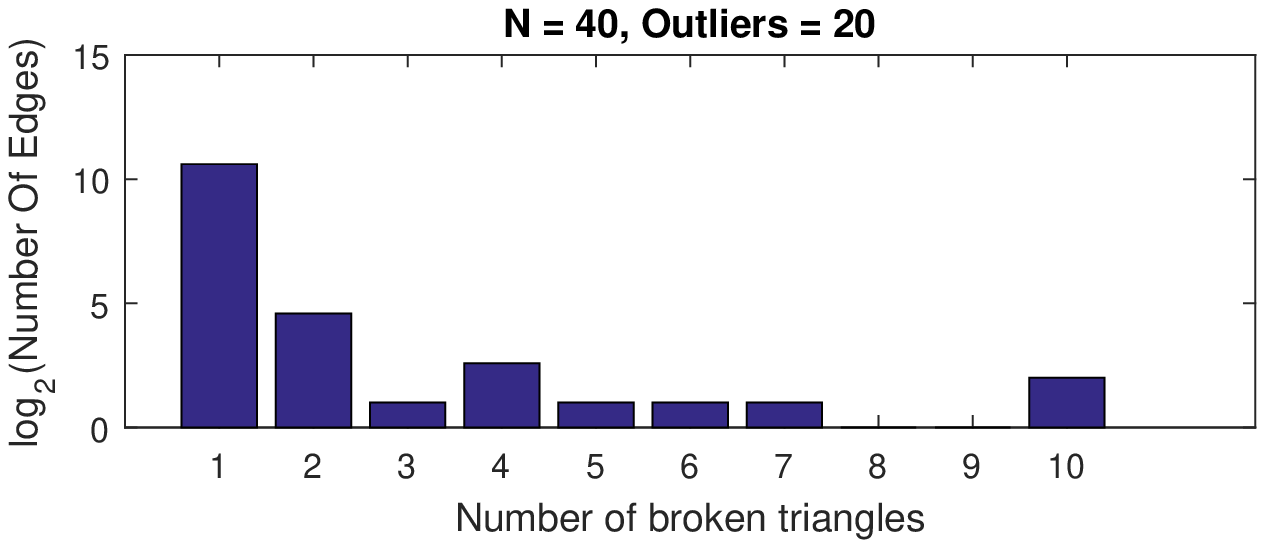} 			
 		}
 	\end{center}
 	\caption{Histogram H(b) counts the number of edges that break 'b' triangles. It can be seen that most of the edges break
         only a few triangles. The tail of the histogram is associated with outliers. The y-axis is logarithmic to better perceive the
         variance.}
 	\label{fig:histogram}
 \end{figure}
 
We wish to determine a threshold $\phi$ to classify the outlier. That is,
an edge which participates in more than $\phi$ broken triangles is classified as an ourlier.
The threshold $\phi$ cannot be set to accurately classify the inlier/outlier edges. Instead, we propose a simple way to determine
this threshold by analyzing the histogram $H$.
We set $\phi$ to be the smallest value that satisfies the following two requirements:

 \begin{enumerate}
 	\item\label{eq:phi_select1} $ \sum_{b=1}^{\phi} H(b) \geq |E| / 2 $
 	\item\label{eq:phi_select2} $ H(\phi + 1) > H (\phi) $

 \end{enumerate}

The first requirement assures that most of the edges are not considered as outliers. This assumption holds in most cases,
but can be adjusted according to the problem setting.
The second requirement corresponds to the observation that outlier edges tend to form a high bin along the tail of the histogram $H$ 
(Figure \ref{fig:histogram}). This simple heuristic performs well empirically (see Section \ref{sec:analysis}).

After the threshold is selected, we can remove the associated distances from the data,
and use the remaining distances to compute an embedding using MDS. A high-level pseudo code of TMDS is described in Algorithm \ref{alg:tmds}.

\begin{algorithm}
    \emph{Let $D$ be a dissimilarity $NxN$ matrix.}\;

    1. Calculate $Count_{ij}$ - the number of broken triangles where the edge $D_{ij}$ participates.

    2. Calculate the histogram $H(b)$, where the bin $H(b)$ counts the number of edges that participate in exactly $b$ broken triangles.

     3. Find the threshold $\phi$ according to \ref{eq:phi_select1} and \ref{eq:phi_select2}.
     
     4. Let $F$ be an $N\times N$ matrix where:
    $
    F_{ij} = \left.
  \begin{cases}
    0, & Count_{ij} > \phi \\
    1, & otherwise 
  \end{cases}
  \right\}
    $
    
    5. Execute a weighted MDS with an associated weight matrix $F$.

\caption{TMDS}
\label{alg:tmds}
\end{algorithm}

\section{Analysis}
\label{sec:analysis}

\subsection{Algorithm complexity}

Testing all the triangles to identify the broken ones, amounts to a time complexity of $O(N^3)$, an order of magnitude larger than that
of SMACOF ($O(N^2)$).
To avoid increasing the total time complexity of the MDS method, we can subsample $O(N^2)$ triangles and build the histogram based only
on them. We can use uniform sampling, where, for every edge $D_{ij}$, we sample a constant number of points to
form a constant number of triangles.
As can be observed in Figure \ref{fig:sampling_triangles}, testing too few triangles, may impair the detection rate.
It can be seen that 45 triangles per edge are enough to detect most of the outliers, and that it scales well with N. 
Empirically, we observed that sampling twice as many triangles as the expected number of outliers is an effective rule-of-thumb. 
In practice, for $N = 100$  TMDS without any sub-sampling takes only $2$ seconds using a non-optimized
implementation in Matlab, while computing the embedding itself using SMACOF takes $1.9$ seconds. This suggests that the filtering step does not incur significant overhead. See Table \ref{table:benchmark}.

\begin{table}[t]
\centering
\begin{tabular}{llll}
\hline
N &  SMACOF    & TMDS  &  FG12  \\ 
\hline
150   & 2      & 4.84 & 15.9  \\
300   & 3.5    & 9.35 & 59.4 \\
450   & 5.5    & 19.6 & 105.5  \\
600   & 8.9    & 32.2 & 207.8  \\
750   & 13.5   & 46.9 &  309.4  \\
900   & 14.8   & 59.9 & 491.4  \\

\hline
\end{tabular}
\caption{CPU time in seconds of three embedding algorithms of data sampled uniformly from a 2d unit hypercube with a various number ($N$) of objects. The distance matrix was contaminated with $10\%$ outliers. All the methods were implemented in Matlab and tested on a single core of i3-6200U processor. For TMDS we sampled 100 triangles per edge, and for FG12 we used $\lambda=2$. 
TMDS is faster than FG12 and compared to SMACOF has a close to constant multiplicative overhead, as expected.
}
\label{table:benchmark}
\end{table}

\begin{figure}[h]
	\begin{center}
		\setlength\fboxsep{0pt}
		\setlength\fboxrule{0pt}
		\fbox{\rule{0pt}{0in}
			\includegraphics[width=\linewidth]{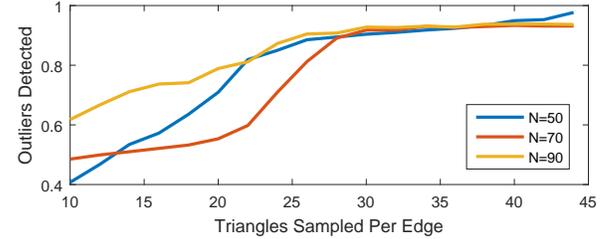}
			
		}
	\end{center}
	\caption{The number of outliers detected as a function of the number of subsampled per-edge triangles. 
        The data consists of 15 outliers. }
	\label{fig:sampling_triangles}
\end{figure}

\subsection{Evaluation}

To evaluate TMDS, we use synthesized data of various magnitudes, dimensions, and portions of outliers. We measure
the precision-recall performance of detecting the outliers, and the quality of the embedding with and without our outlier filtering.
More precisely, we synthesize ground-truth data by randomly sampling $N$ points in a $d$ dimensional hypercube,
and compute the pairwise distances $D$ between them. We randomly pick $M$ elements and replace them with a random element from
the distance matrix.

\vspace{0.2cm}

\noindent\textbf{Qualitative evaluation.} We use Shepard Diagrams to visually display the classification of data elements as either 
outliers or inliers.
In the diagram in Figure \ref{fig:shepard_tmds}, each point represents a distance. The X-axis represents the input distances and the
Y-axis represents the distance in the embedding result. Points on the main diagonal are inliers representing the distances that are
correctly preserved in the embedding. The red circles represent the distances that TMDS detects as outliers. The blue
off-diagonal points are the false negatives, and the red dots on the diagonal are false positive distances. The number of false
positives and negatives increases as the number of outliers increases.

\begin{figure}[b]
	\centering
	\begin{subfigure}[b]{0.17\textheight}
		\includegraphics[trim={1cm 0 1cm 0},clip,width=\linewidth]{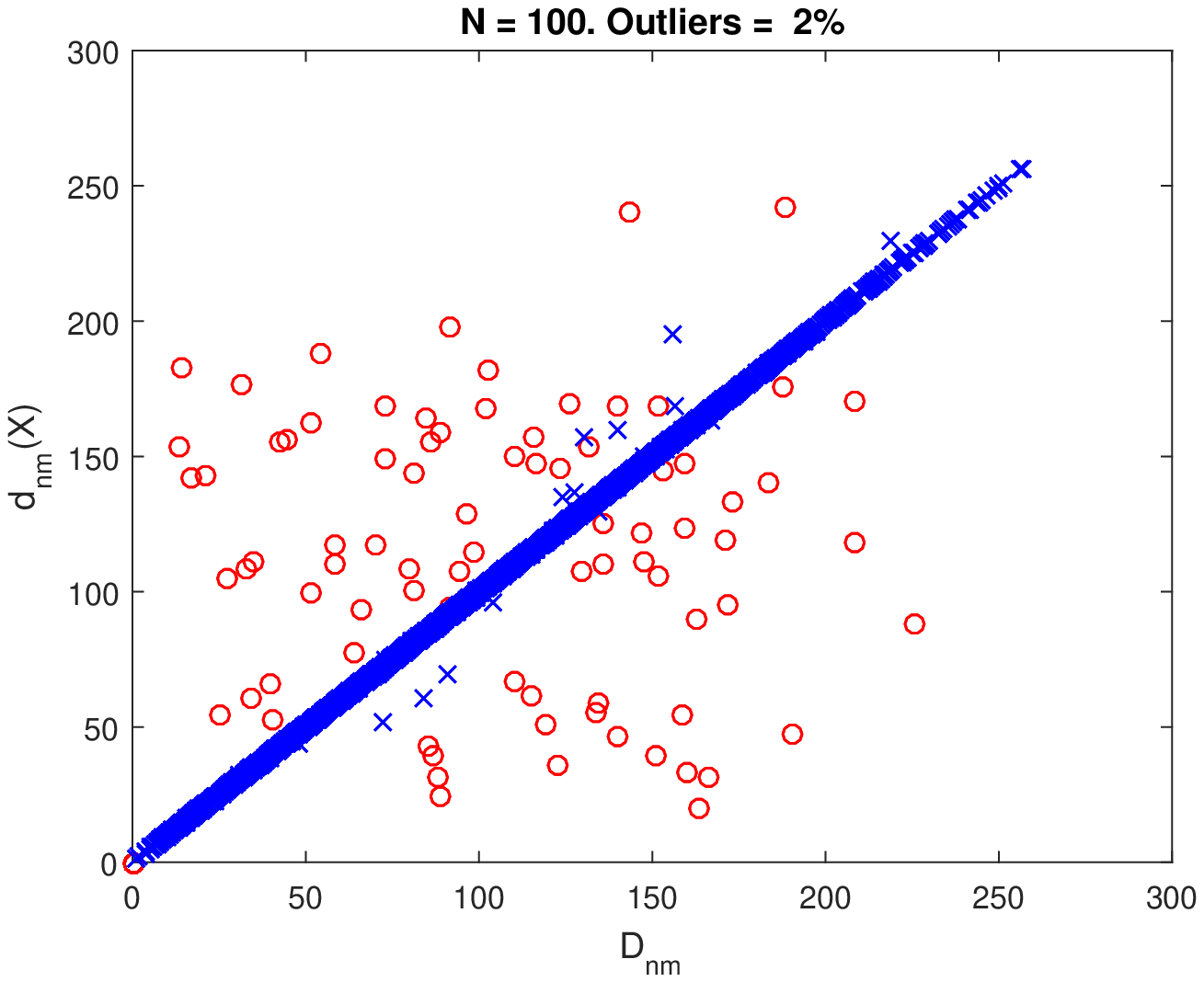}
		\caption{}
	\end{subfigure}
	\begin{subfigure}[b]{0.17\textheight}
		\includegraphics[trim={1cm 0 1cm 0},clip,width=\linewidth]{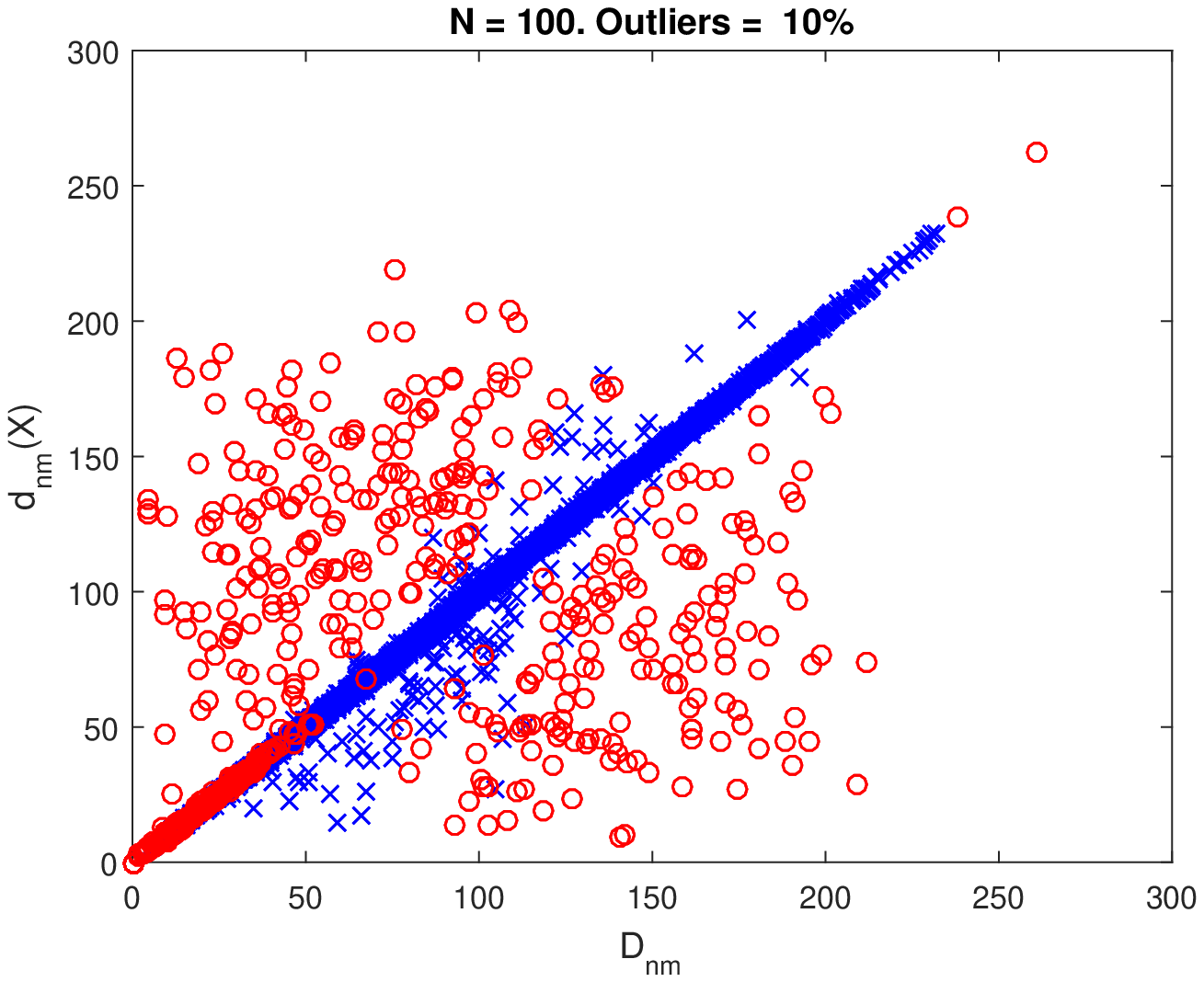}
		\caption{}
	\end{subfigure}

	\caption{Shepard Diagram. Each point represents a distance. The X-axis represents the input distances and the Y-axis represents 
        the distance in the embedding result. The red circles represent the edges that are considered as outliers. 
        (a) $2\%$ outliers. (b) $10\%$ outliers.}
	\label{fig:shepard_tmds}
\end{figure}

\vspace{0.2cm}

\noindent\textbf{Quantitative evaluation.}
We employ two models to quantitatively evaluate TMDS.
In the first, we select outliers at random, while in the second all edges are distorted by a log-normal distribution.

First, we measure the accuracy of outlier detection using precision-recall. Each distance is classified as either inlier or outlier, and
the detection can thus be regarded as a retrieval process, where precision is the fraction of retrieved outliers that are true
positives, and recall is the fraction of true positives that are detected. As can be observed in Figure \ref{fig:precision_recall},
our precision and recall are high, where the first moderately decreases with the number of outliers and the latter increases.
The precision is larger than $75\%$, which implies that a non-negligible portion of the filtered distances are false positives. However,
the excess of filtering is not destructive since MDS is an over-determined problem. Yet, at some point, filtering too many distances
impairs the embedding, as will be demonstrated in a separate experiment below.

\begin{figure}[h]
	\begin{center}
		\setlength\fboxsep{0pt}
		\setlength\fboxrule{0pt}
		\fbox{\rule{0pt}{0in}
			\includegraphics[width=\linewidth]{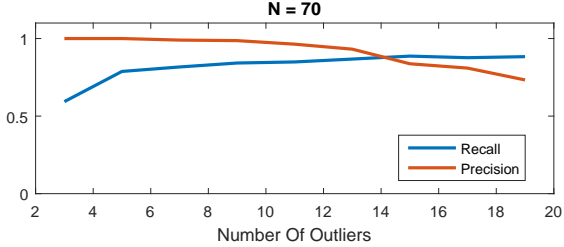}
			
		}
	\end{center}
	\caption{Precision and Recall plot. N=70.}
	\label{fig:precision_recall}
\end{figure}

\vspace{0.2cm}

\noindent\textbf{The detection probability.} 
Figure \ref{fig:drastic_edge_dim2} shows the probability of an outlier to be detected as a function of its error magnitude, which is 
measured by the ratio between the actual distance $d_{out}$ and the true distance $d_{GT}$.
As can be observed, edges that are strongly deformed (either squeezed or enlarged) are likely to be detected. This holds also for 
higher dimensions.

\begin{figure}[h]
	\begin{center}
		\setlength\fboxsep{0pt}
		\setlength\fboxrule{0pt}
		\fbox{\rule{0pt}{0in}
			\includegraphics[width=\linewidth]{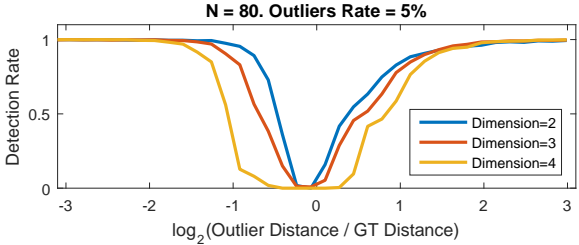}			
		}
	\end{center}
	\caption{The outlier detection rate as a function of the shrinkage \ enlargement of the outliers relative to the ground-truth value. 
    Edges that are strongly deformed (either squeezed or enlarged) are likely to be detected. Note that the X-axis is logarithmic: $ log_2 ({d_{out}}/{d_{GT}})$.}
	\label{fig:drastic_edge_dim2}
\end{figure}

\vspace{0.2cm}

\noindent\textbf{Embedding evaluation.} 
To obtain insight about the embedding performance of points $X_1,...,X_N$, we used the following score:
$$ S{ij} = \biggm\lvert log\frac{|| X_i - X_j ||}{D_{ij}} \biggm\lvert $$
Then, we take the average of $S_{ij}$ as the score for the embedding.
This scoring treats shrinkage and enlargements equally, where a low score implies a better embedding.
The results of the evaluation are displayed in Figure \ref{fig:mds_perf}. The plot shows that when the portion of outliers is less 
than $ 22\% $, our pre-filtering performs better than applying SMACOF MDS directly. A larger amount of outliers causes TMDS 
to filter out too many inliers and impairs the embedding.

\begin{figure}[b]
	\begin{center}
		\setlength\fboxsep{0pt}
		\setlength\fboxrule{0pt}
		\fbox{\rule{0pt}{0in}
			\includegraphics[width=\linewidth]{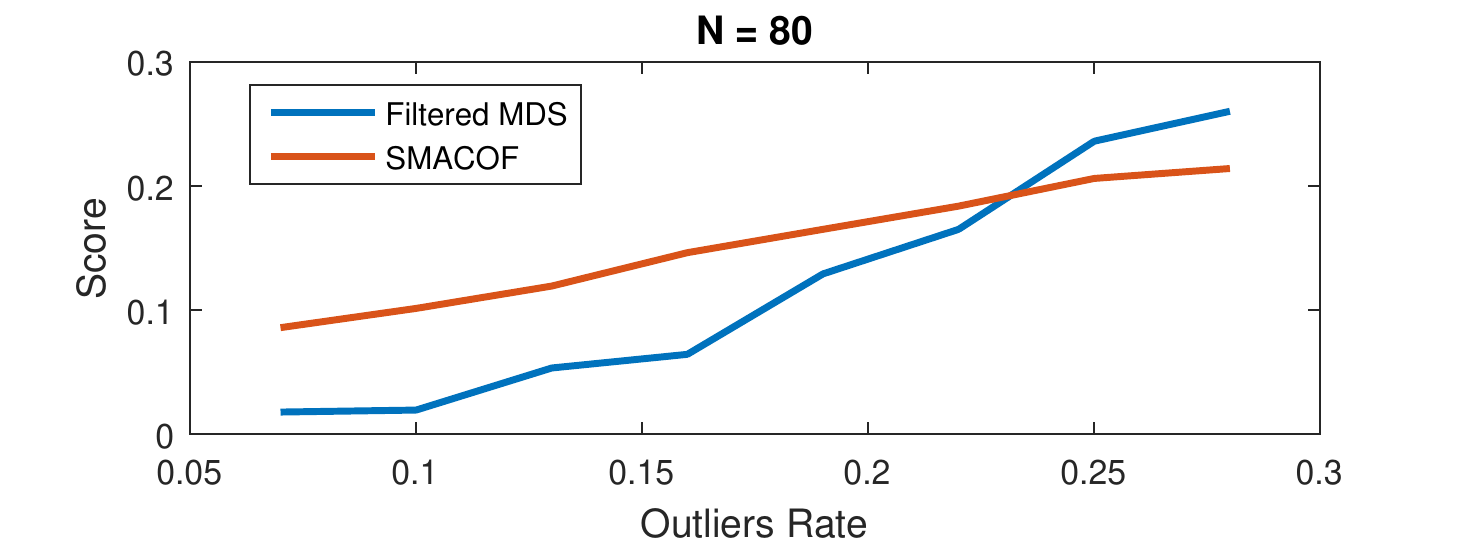}
			
		}
	\end{center}
	\caption{A comparison between SMACOF and TMDS as a function of outlier rate. Up to $ 22\% $ TMDS has better 
        performance.}
	\label{fig:mds_perf}
\end{figure}

\noindent\textbf{Log-normal distribution.}
We generated synthetic data that mimics realistic data characteristics, by sampling $N$ data points uniformly in a $d$-dimensional 
hypercube, and forming the respective distance matrix $D$. We distorted the distances using factors of log-normal distribution. 
That is, every distance $D_{ij}$ is multiplied by a factor sampled from a log-normal distribution, where the log-normal mean is 1. 
Note that these distorted distances include both noise and outliers. The results of those simulations with different log standard 
deviation $\sigma$ are presented in Figure \ref{fig:lognormal_1}. Note that for larger $\sigma$ values, which signify larger errors,
the effectiveness of TMDS is more significant.

\begin{figure}[t]
	\begin{center}
		\setlength\fboxsep{0pt}
		\setlength\fboxrule{0pt}
		\fbox{\rule{0pt}{0in}
			\includegraphics[width=\linewidth]{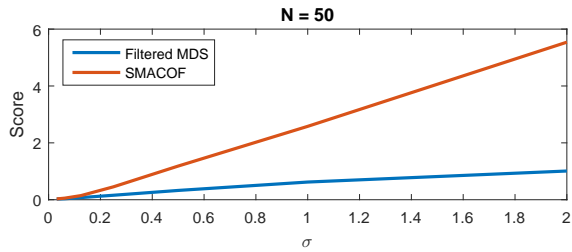}
			
		}
	\end{center}
	\caption{A comparison between SMACOF and TMDS for various distributions defined as a function of $\sigma$.}
	\label{fig:lognormal_1}
\end{figure}

\noindent\textbf{Non-uniform Distributions.}
We further evaluated TMDS on non-uniform structured data.
The embedding of data sets with clear structures are shown in Figure \ref{fig:plus}.
As we shall discuss below, one of the limitations of TMDS is handling straight lines. This is evident in Figure \ref{fig:plus},
where our embedding is imperfect (b), albeit better than without filtering (a).
In (c-d), the structured data is numerically easier for filtering. As demonstrated, the accuracy of the embedding of the data with $15\%$ outliers is high.

\begin{figure}[h]
	\centering
	\begin{subfigure}[b]{0.17\textheight}
		\includegraphics[trim={1cm 0 1cm 0},clip,width=\linewidth]{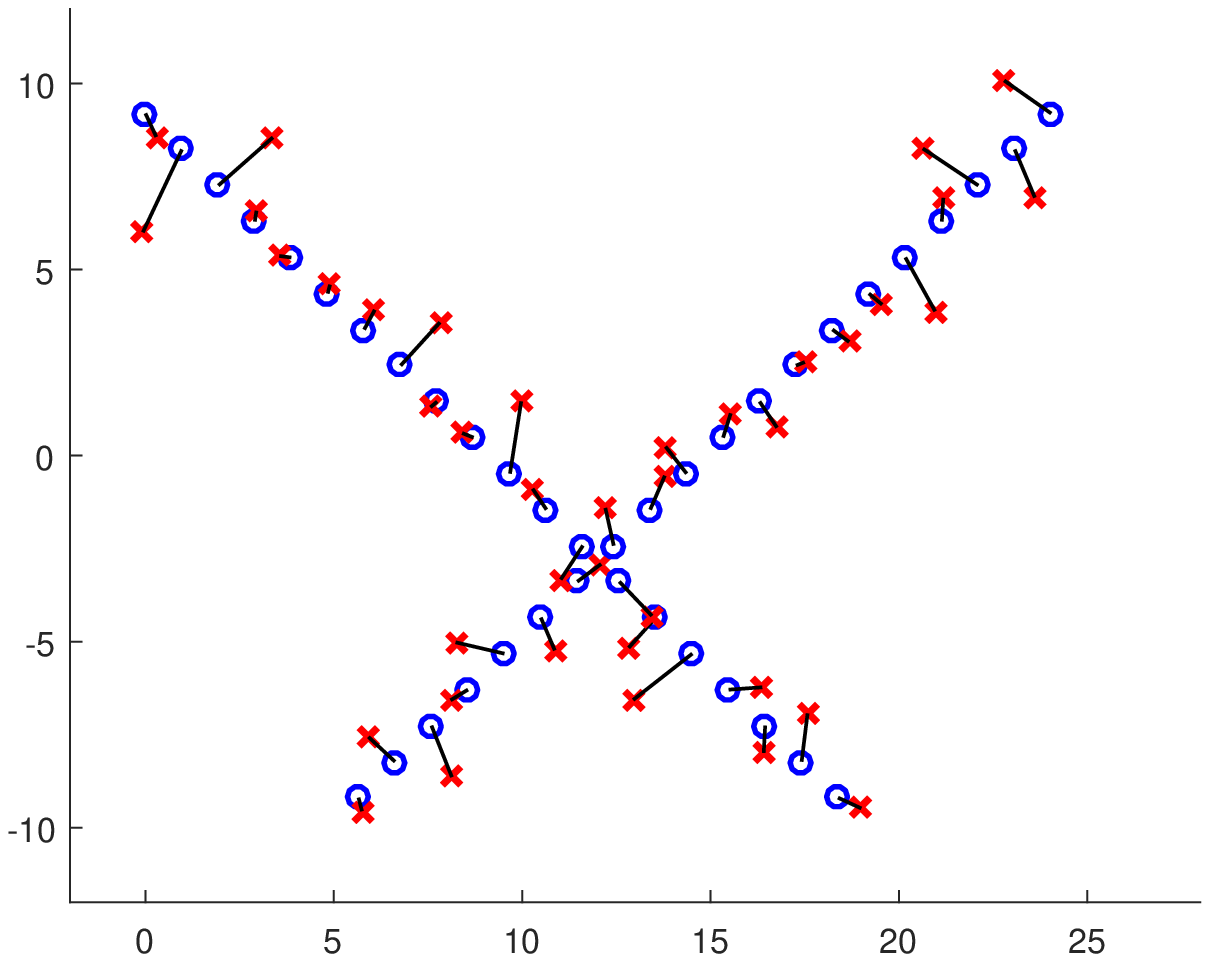}
        \caption*{SMACOF}
	\end{subfigure}
	\begin{subfigure}[b]{0.17\textheight}
		\includegraphics[trim={1cm 0 1cm 0},clip,width=\linewidth]{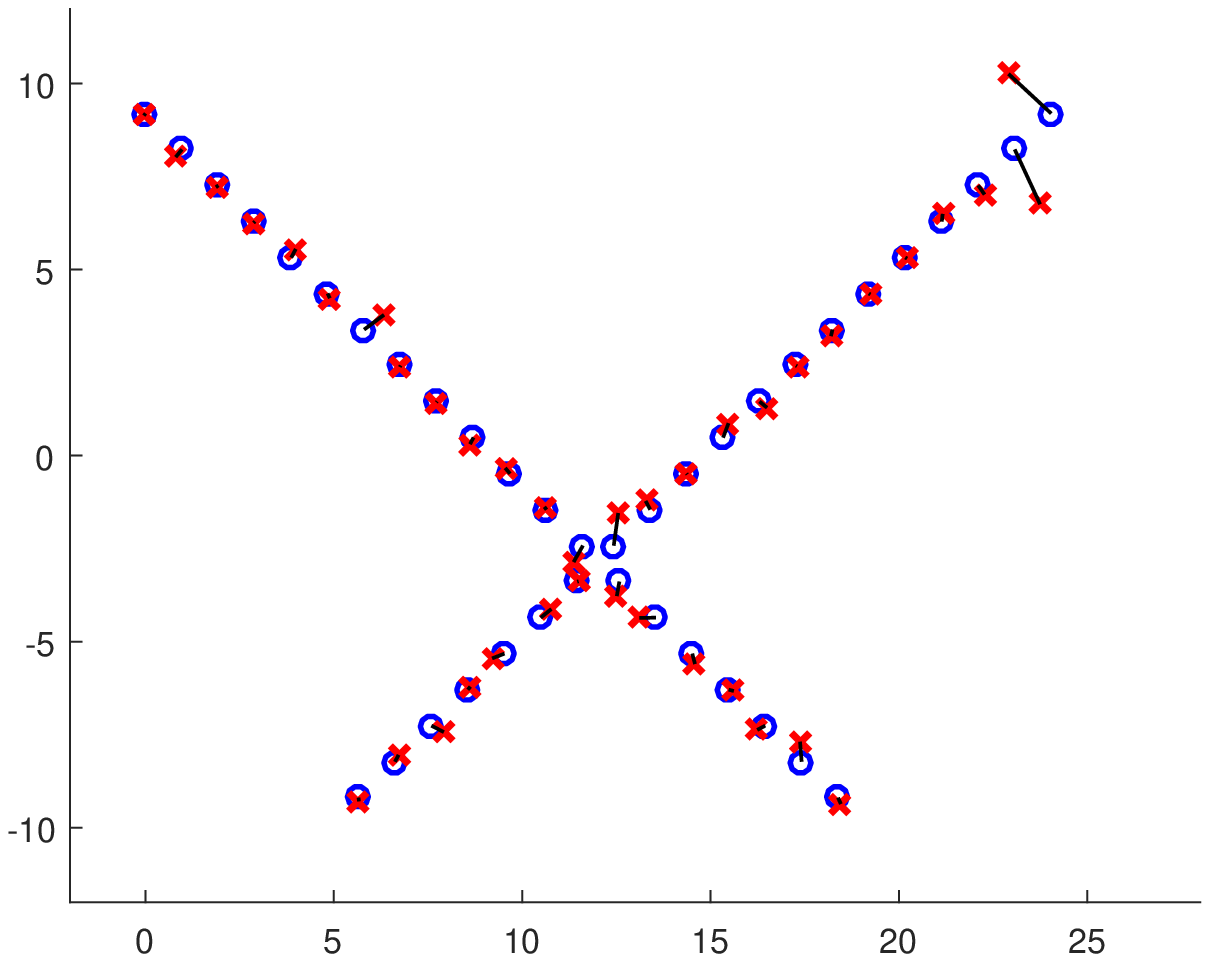}
        \caption*{TMDS}
	\end{subfigure}
\begin{subfigure}[b]{0.17\textheight}
		\includegraphics[trim={1cm 0 1cm 0},clip,width=\linewidth]{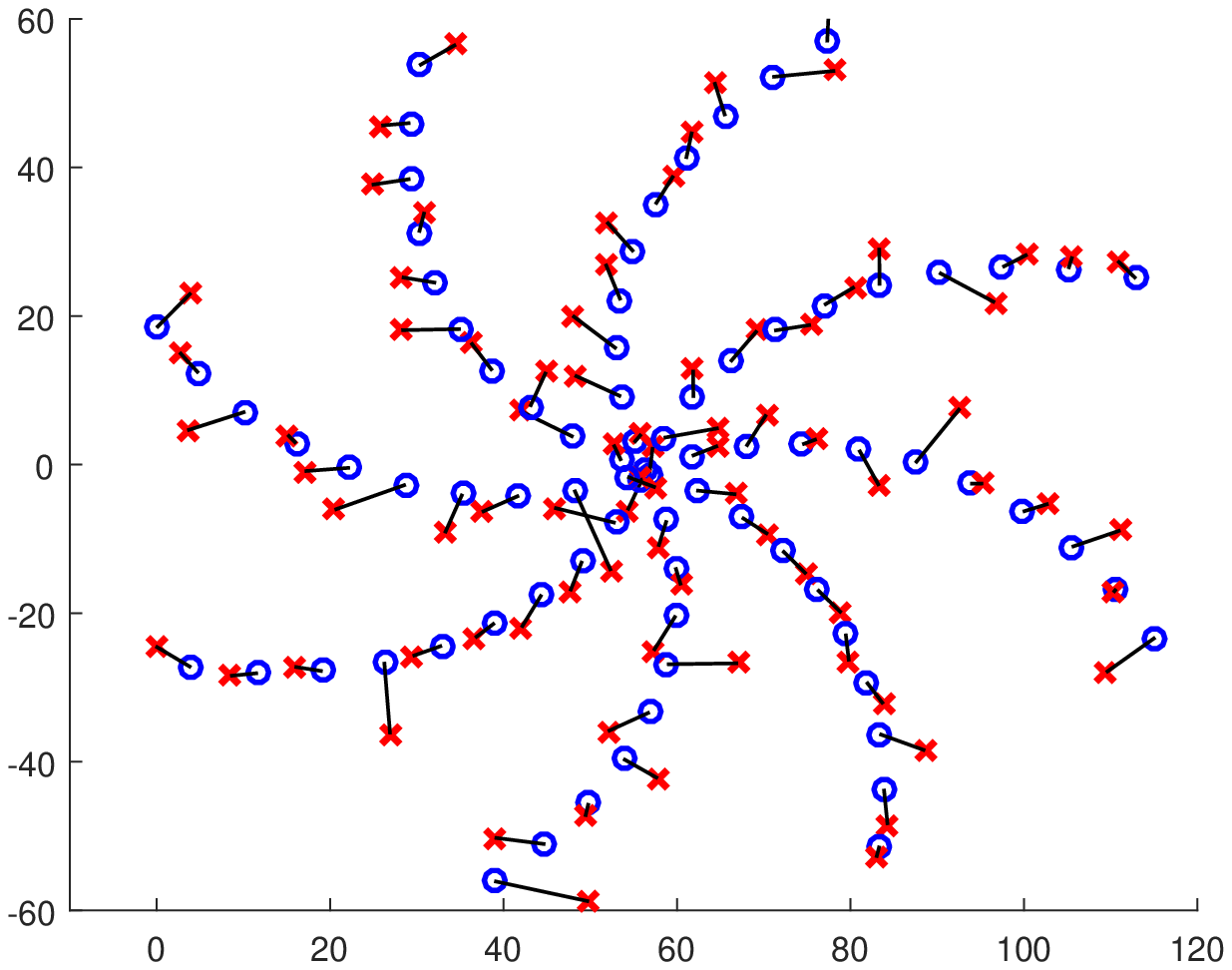}
        \caption*{SMACOF}
	\end{subfigure}
	\begin{subfigure}[b]{0.17\textheight}
		\includegraphics[trim={1cm 0 1cm 0},clip,width=\linewidth]{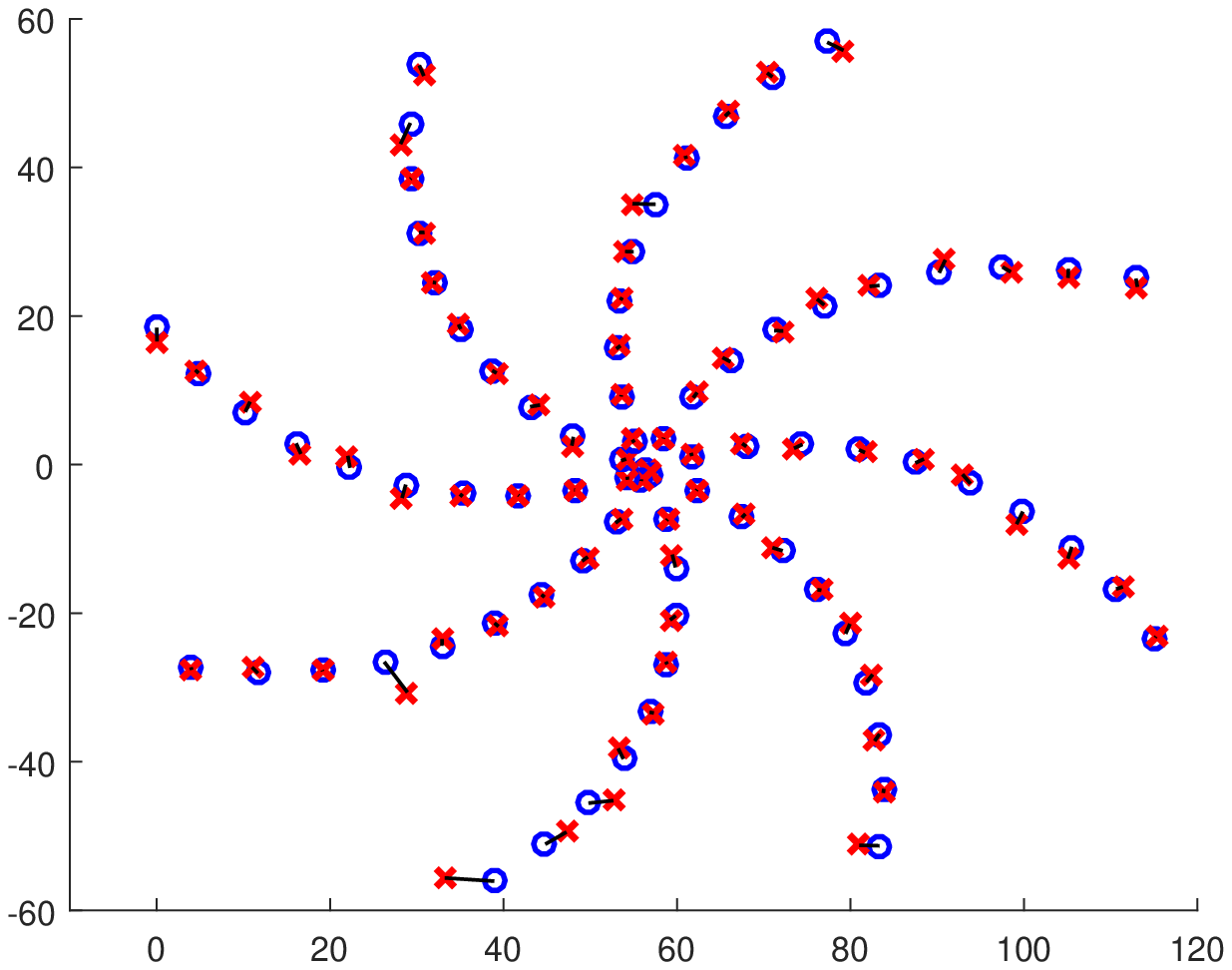}
        \caption*{TMDS}
	\end{subfigure}
	
	\caption{The embedding of a 'PLUS' shaped dataset (upper row) with $10\%$ outliers, and a 'SPIRAL' shaped dataset (bottom row) 
        with $15\%$ outliers. Each embedding point is connected to its ground-truth point.}
	\label{fig:plus}
\end{figure}

\subsection{Limitations}
As described, our algorithm consists of two parts - building the histogram of broken triangles and setting the threshold.
Each part has it own limitations. The analysis of the histogram assumes that outlier edges break many triangles, while inliers do not.
This assumption may not hold when a large number of triangles are broken also by inlier distances. This may happen when
$D_{ij} = D_{ik} + D_{kj}$, but due to numerical issues, $D_{ij} > D_{ik} + D_{kj}$, for example, when many points reside along 
straight lines.
Setting the threshold $\phi$ is merely a heuristic that may fail when there are too many outliers, or too few that have a special
distribution that confuses the heuristics and causes half of the data to be considered as outliers.

\section{Experiments}
A Matlab implementation is available at \url{goo.gl/399buj}.


In the previous sections, we have tested and evaluated our method (TMDS) using synthetic data for which we have the ground truth distances. We have shown that the MDS technique of Forero and Giannakis \cite{forero2012sparsity} is sensitive to the initial guess and its performance is dependent on a user-selected parameter. In the following, we evaluate our method on data for which the ground truth is not available. To evaluate the performance, we compare TMDS with a SMACOF implementation of MDS using numerous common measures that assume ground truth labels, but not distances. The labels are used for evaluation only. First, we test TMDS on two real datasets available on the Web, for which the ground truth distances are available. Next, we evaluate TMDS on three datasets for which ground truth distances are unavailable, but classes labels exist.

\textbf{SGB128 dataset}
This dataset consists of the locations of 128 cities in North America \cite{CitiesDataset}. We compute the distances among the cities, and introduce $15\%$ outliers.
Figure \ref{fig:US128} shows the ground-truth locations, the embedding obtained by SMACOF on the outlier data and the result of 
SMACOF after applying our filtering technique. As can be seen, applying the filtering prior to computing the embedding yields a significant improvement.

\begin{figure}[b]
	\centering
	\begin{subfigure}[b]{0.17\textheight}
		\includegraphics[trim={1cm 0 1cm 0},clip,width=\linewidth]{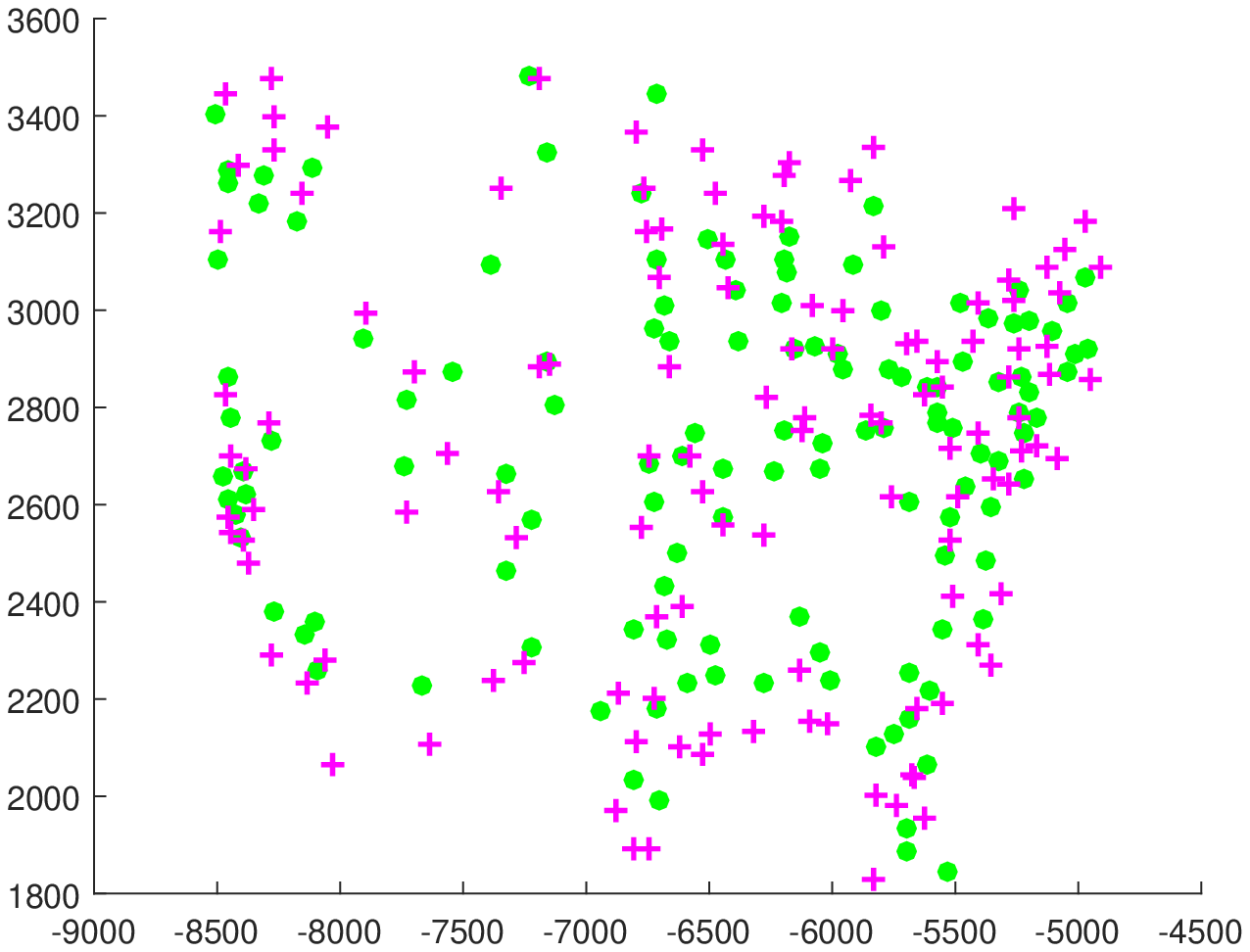}
		\caption*{SMACOF}
	\end{subfigure}
	\begin{subfigure}[b]{0.17\textheight}
		\includegraphics[trim={1cm 0 1cm 0},clip,width=\linewidth]{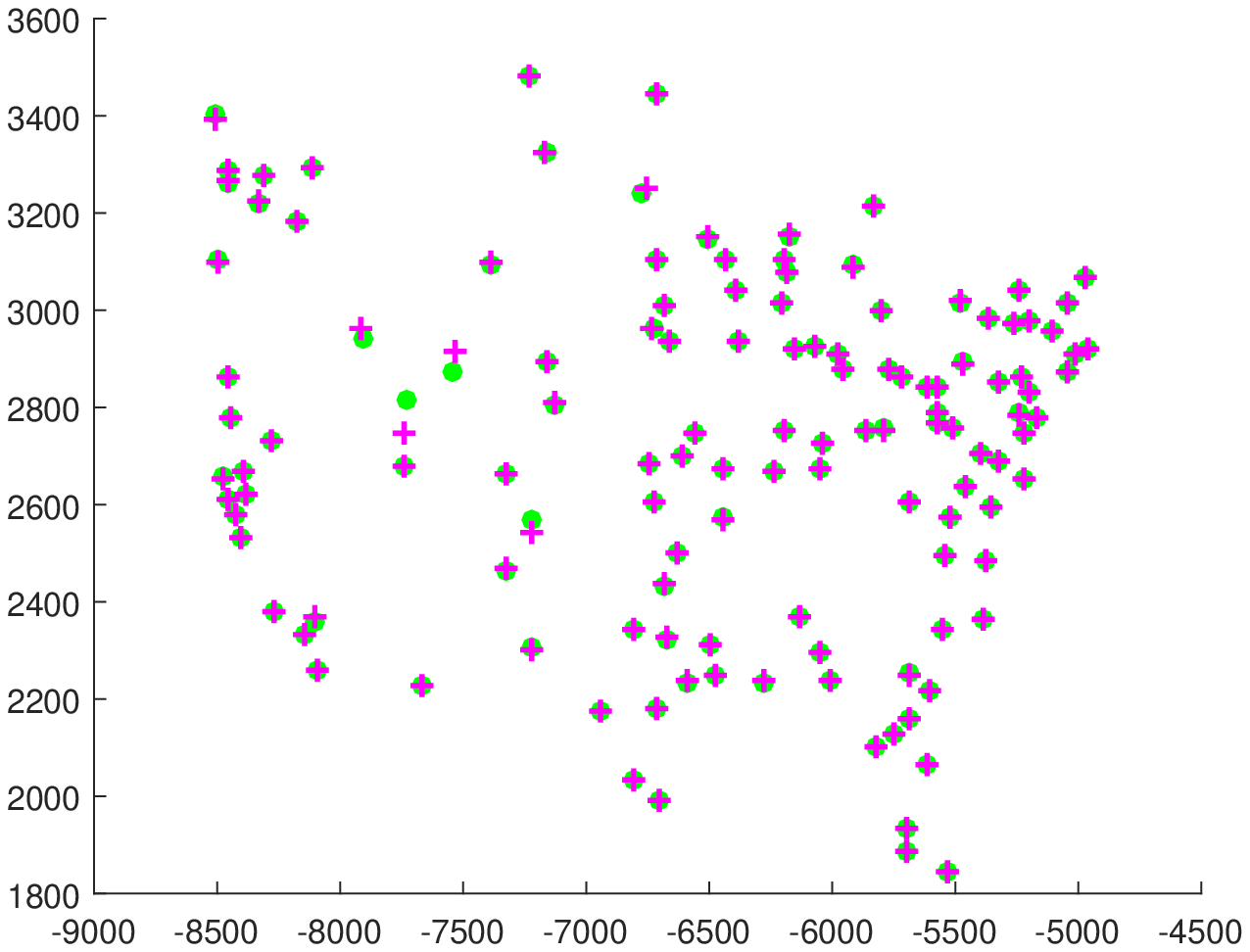}
		\caption*{TMDS}
	\end{subfigure}
	\caption{Two-dimensional embedding of SGB128 distances with $10\%$ outliers.
		The green dots are the ground-truth locations and the magenta dots represent the embedded points.}
	\label{fig:US128}
\end{figure}

\textbf{Protein dataset} 
This dataset consists of a proximity matrix derived from the structural comparison of 213 protein sequences 
\cite{denoeux2004evclus, graepel1999classification}.These dissimilarity values do not form a metric. Each of these proteins is associated with one of four classes: hemoglobin-α (HA), hemoglobin-β (HB), myoglobin (M) and heterogeneous globins (G). We embed the dissimilarities using SMACOF with and without 
our filtering. We then perform $k$-means clustering and evaluate the clusters with respect to the ground-truth classes. The results
are shown in Figure \ref{fig:protein}. The clustering results are evaluated using four common measures: ARI, RI and MOC. As can be seen,
TMDS outperforms a direct application of SMACOF by all the measures.

\begin{figure}[t]
    \centering
	\includegraphics[trim={1cm 0 1cm 0},clip,width=0.9\linewidth]{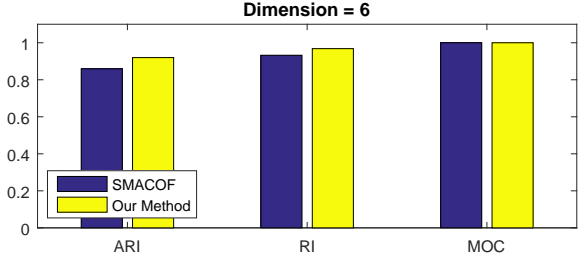}

	\caption{Average cluster index value of 10 executions. The embedding dimension is set to 6, since for lower dimensions
        SMACOF fails due to co-located points.}
	\label{fig:protein}
\end{figure}

\section{Graphical Datasets}

In the following, we conduct experiments to show the effectiveness of our robust MDS for visualization of graphical elements. Specifically, we use three datasets, one of 2D shapes taken from the MPEG7 dataset \cite{ling2005using} and 1070db \cite{1070shapes} shape dataset, one of motion data (CMU Motion Capture Database \url{mocap.cs.cmu.edu}) and one of 3D shapes \cite{Chen2009ABF}. 
For each of the sets we use an appropriate distance measure to generate an all-pairs distance matrix, and then create a map by embedding the elements using SMACOF MDS and our robust MDS. More specifically, we use inner distance to measure distances between 2D shapes, LMA-features \cite{Aristidou:2015_CGF} for measuring distances between sequences, and SHED \cite{kleiman2015shed} to measure the distances among 3D shapes.

We evaluate the embedding qualitatively and quantitatively. Since the ground truth is unavailable, we use the known classes for the evaluation. We expect a good embedding to separate the classes into tight clusters, and specifically avoid intersections among them. Qualitatively, this can be observed in the three side-by-side comparisons of the SMACOF vs. TMDS embedding in Figures \ref{fig:mpeg7}, \ref{fig:shed_mds}, \ref{fig:mpeg7_imperfect},  \ref{fig:mpeg7_with_lines}, and \ref{fig:1070db_10class}.
To measure it quantitatively, we use two common measures, namely Silhouette Index \cite{rousseeuw1987silhouettes} and Calinski Harabasz \cite{calinski1974dendrite}, which analyze how well the different classes are separated from each other. We also use common measures to analyze how well a clustering technique succeeds in clustering the data. We apply K-means and measure its success using AMI\cite{vinh2010information}, NMI \cite{strehl2002cluster}, and Completeness and Homogeneity measures \cite{rosenberg2007v}. As can be noted in Table \ref{table:andreas_score}, our TMDS method outperforms SMACOF in all of these measures.

\begin{figure*}[t]
	\centering
	\begin{subfigure}[b]{0.35\textheight}
		\includegraphics[trim={1cm 0cm 2cm 1cm},clip,width=\linewidth]{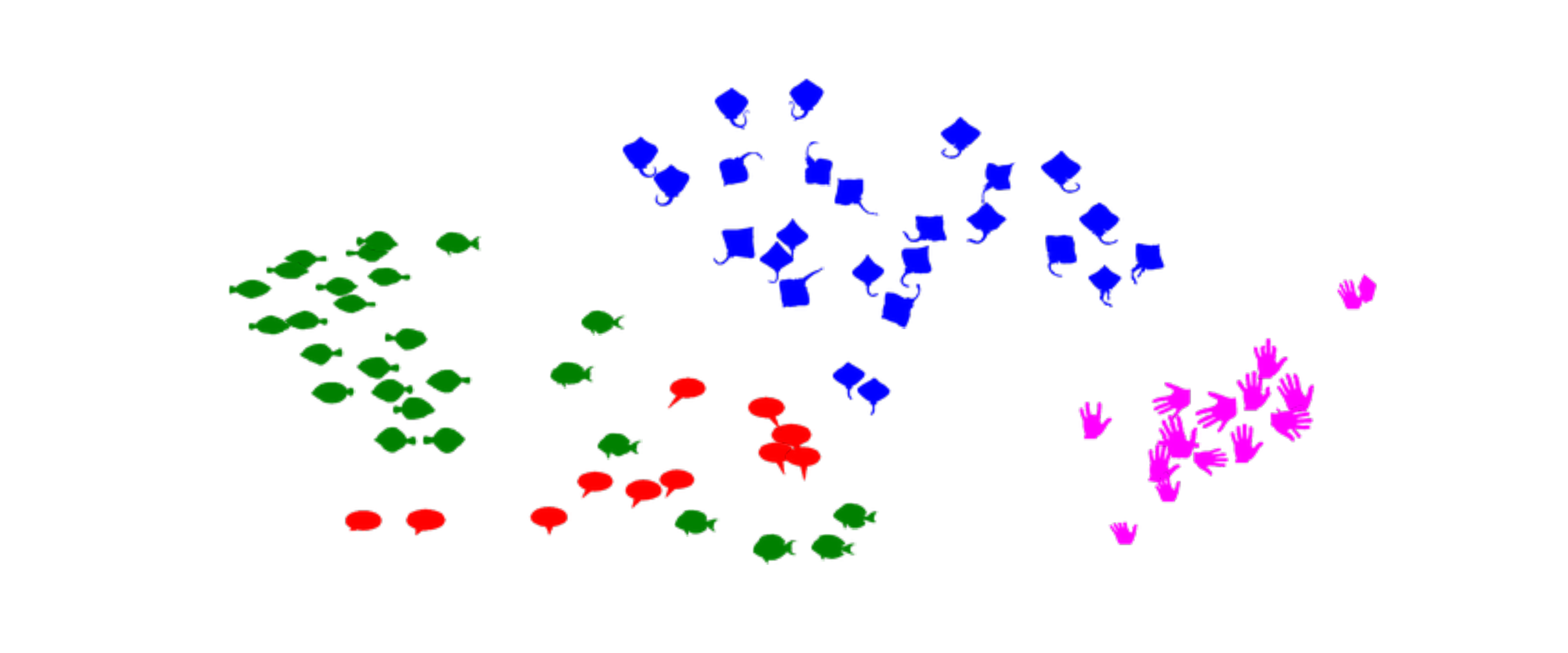}
		\caption{}
	\end{subfigure}
    \begin{subfigure}[b]{0.35\textheight}
	    \includegraphics[trim={1cm 0cm 2cm 1cm},clip,width=\linewidth]{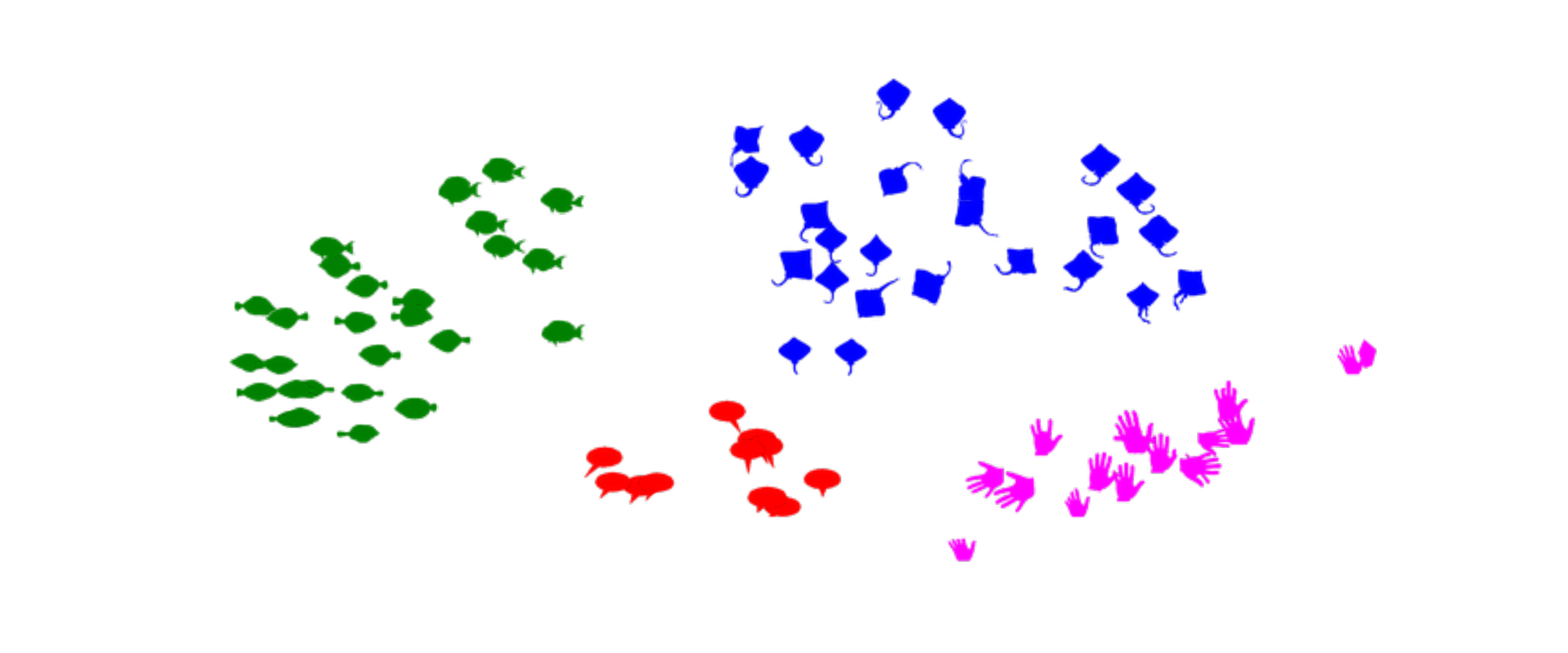}
     	\caption{}
    \end{subfigure}

	\caption{Embedding of shapes from using (a) SMACOF MDS and our Robust MDS (b). As can be seen, the four classes of shapes are better separated using our method. This is also qualitatively supported by the silhouette  measure as we shall elaborate in Section 5.}
	\label{fig:mpeg7}
\end{figure*}

\begin{figure*}[t]
	\centering
	\begin{subfigure}{0.36\textheight}
		\includegraphics[trim={1cm 1cm 1cm 1cm },clip,width=\linewidth]{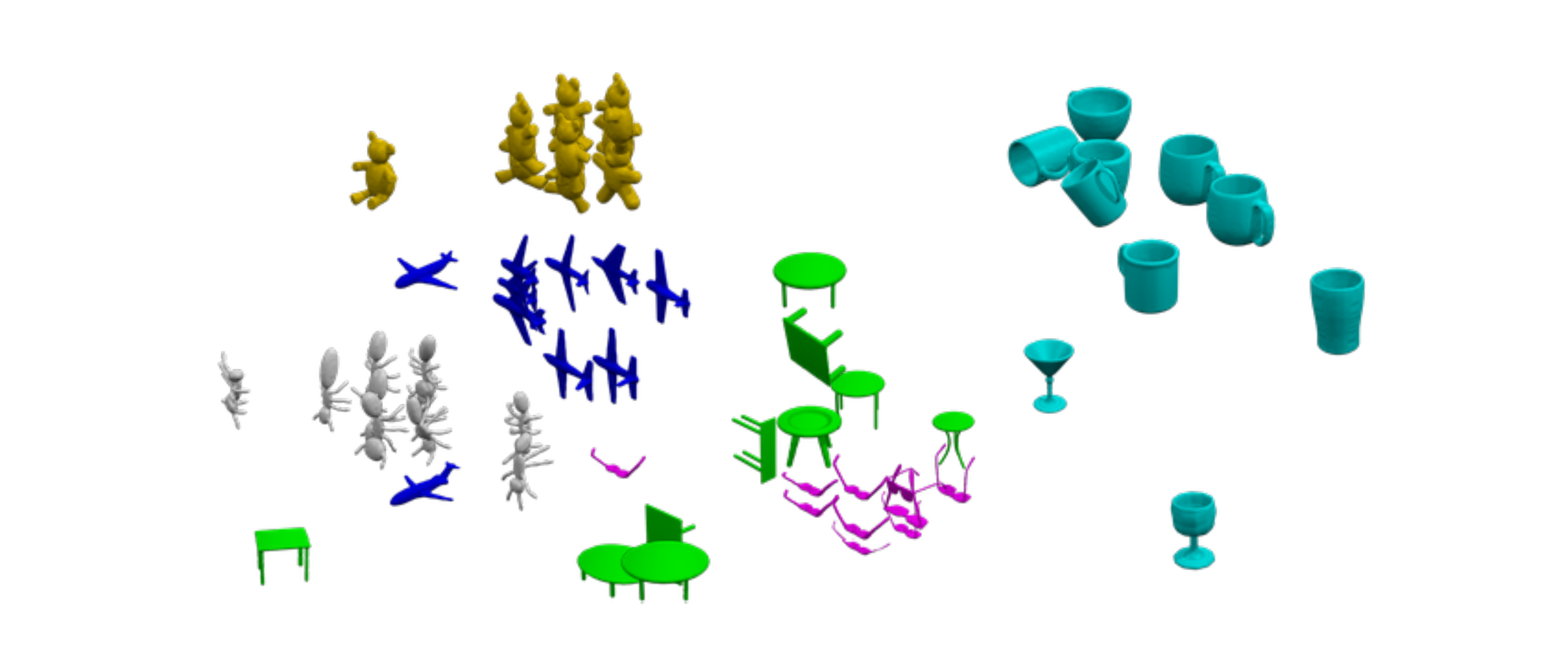}
		\caption*{SMACOF}
	\end{subfigure}
    \begin{subfigure}{0.36\textheight}
		\includegraphics[trim={1cm 1cm 1cm 1cm },clip,width=\linewidth]{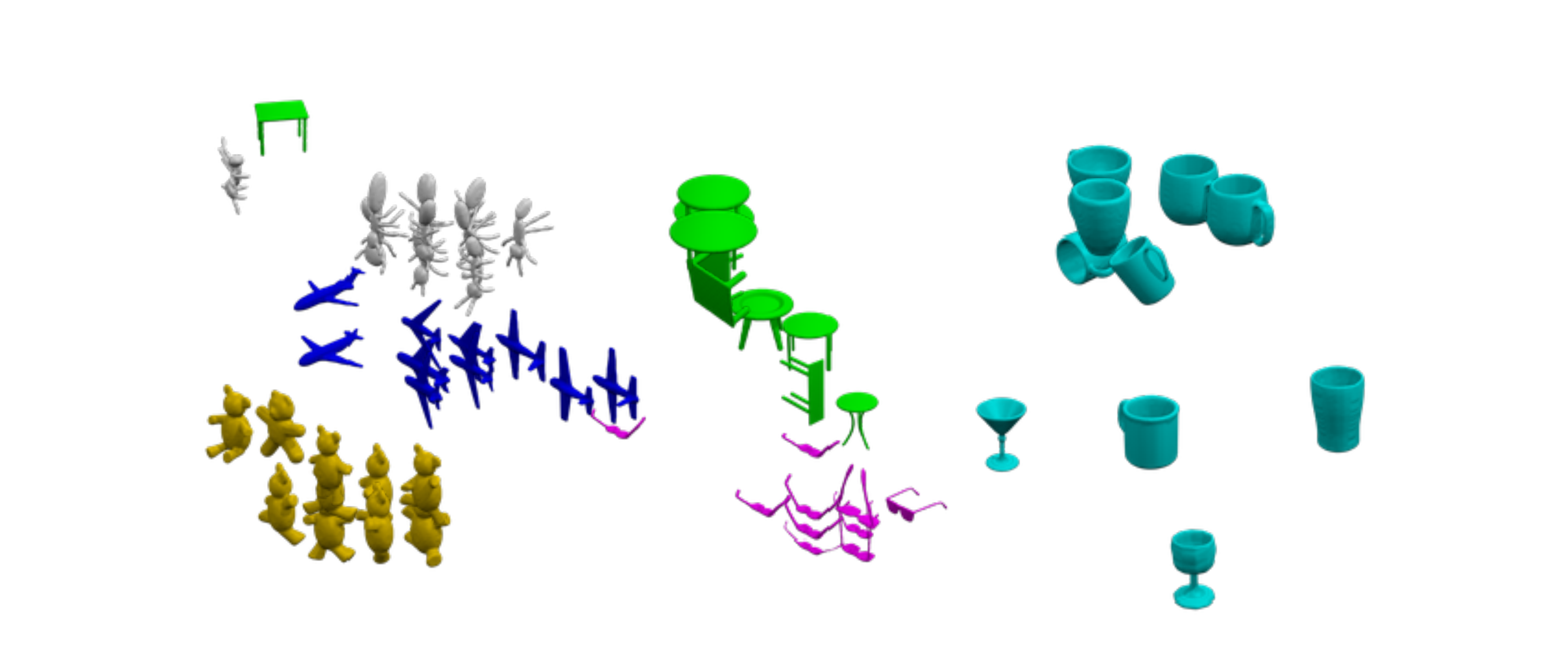}
		\caption*{TMDS}
	\end{subfigure}
	
	\caption{Embedding of 3D shapes dissimilarity matrix. As can be seen, the six classes of shapes are better separated using TMDS. This is also qualitatively supported in Table \ref{table:andreas_score}. }
	\label{fig:shed_mds}
\end{figure*}

The dissimilarity measures we use in all the experiments are not a metric and thus the meaning of outliers should be understood accordingly. The removal of outliers in such a context merely enhances the distances to better conform with a metric. Nevertheless, as can be clearly seen, in all cases, TMDS enhances the maps and provides a better separation of the classes. Still, as can be observed in Figures \ref{fig:mpeg7_imperfect}, \ref{fig:1070db_10class} and \ref{fig:mpeg7_with_lines}, the embedded results are not necessarily perfect, in the sense that there are some elements that are not embedded close enough to the rest of the elements in their respective classes. Note that for completeness, we add results obtained with the method of Forero and Giannakis \cite{forero2012sparsity}, for which we set $\lambda = 2$.

\begin{figure*}[h]
	\centering
	\begin{subfigure}[b]{0.36\textheight}
		\includegraphics[trim={1.5cm 1cm 1cm 1cm },clip,width=\linewidth]{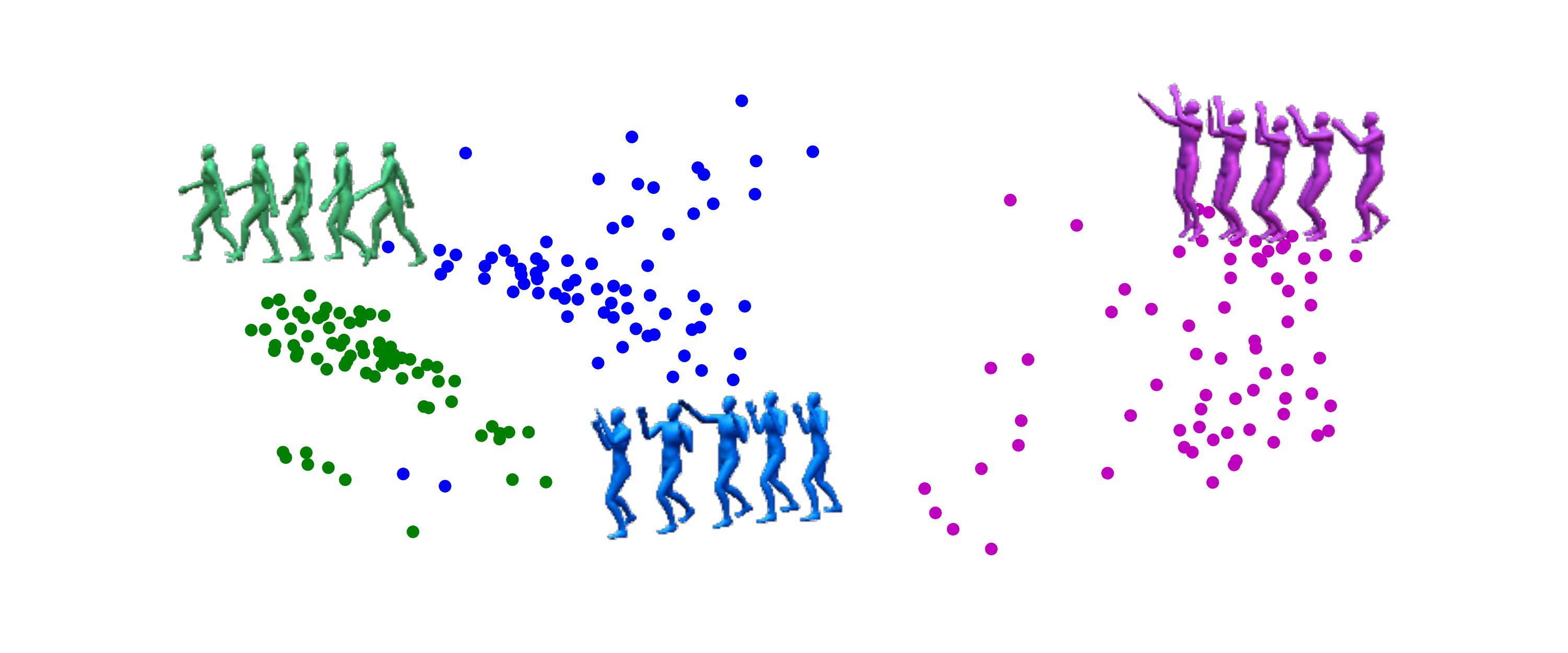}
		\caption*{SMACOF}
	\end{subfigure}
    \begin{subfigure}[b]{0.36\textheight}
		\includegraphics[trim={1.5cm 1cm 1cm 1cm },clip,width=\linewidth]{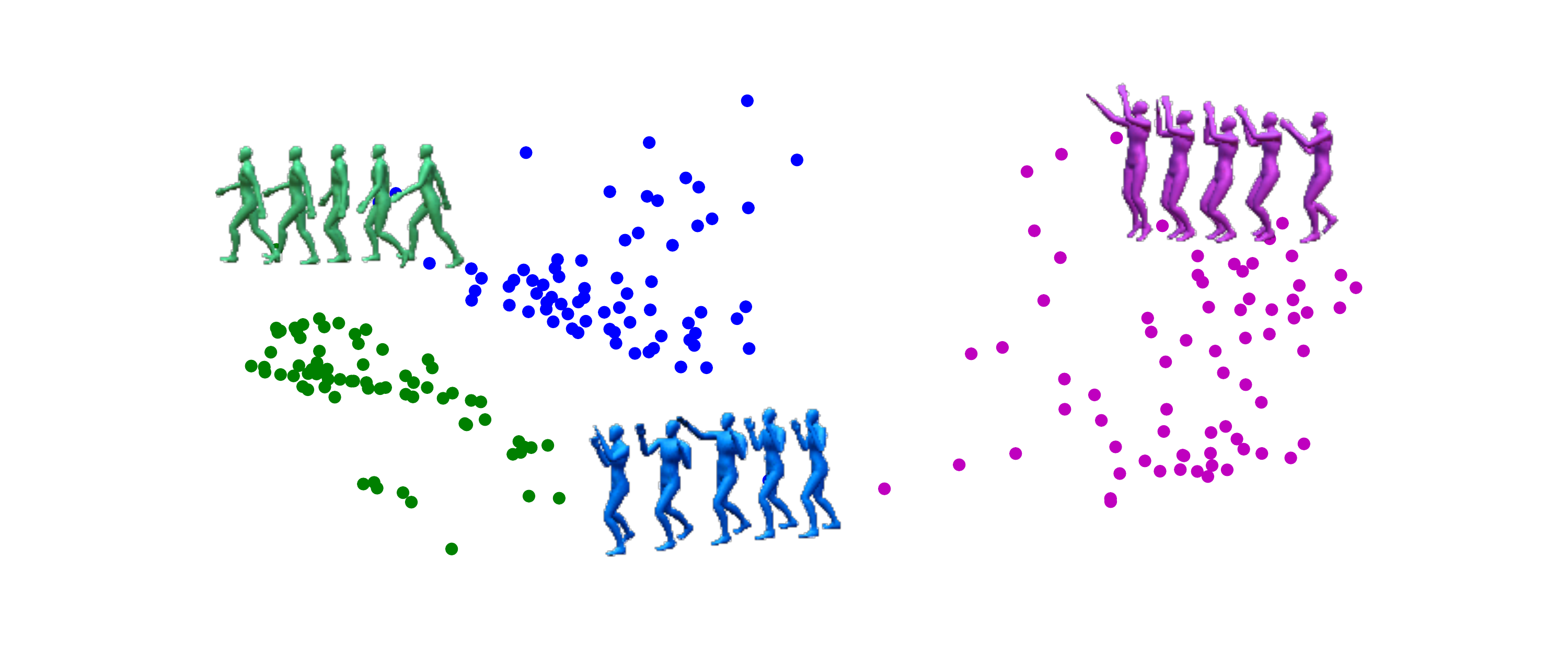}
		\caption*{TMDS}
	\end{subfigure}
	
	\caption{Embedding of the motion data of three classes into 2D space. Note the blue data points that are mixed up with the green ones on the left (SMACOF), versus the better separation on the right (TMDS).} 
	\label{fig:andreas_small_mds}
\end{figure*}

\begin{figure*}[h]
	\centering
	\begin{subfigure}[b]{0.36\textheight}
        \begin{tikzpicture}
          \node(a) at (0,0,0){\includegraphics[trim={1.5cm 0.5cm 1cm 1cm},clip,width=\linewidth]{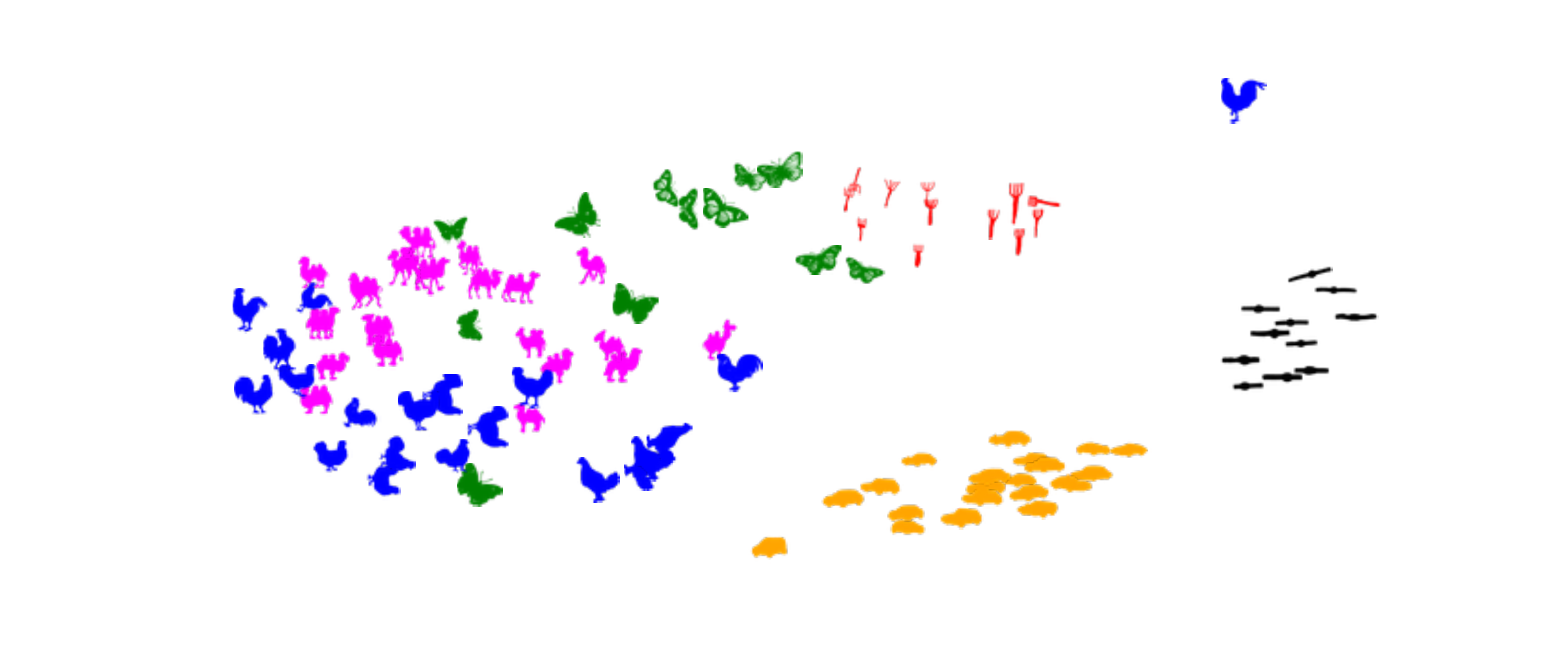}};
          \node at(a.center)[draw, red,line width=1pt,ellipse, minimum width=20pt, minimum height=20pt,xshift=77pt, yshift=39pt]{};
          \node at(a.center)[draw, red,line width=1pt,ellipse, minimum width=20pt, minimum height=20pt,xshift=-55pt, yshift=-28pt]{};
        \end{tikzpicture}
		\caption*{SMACOF}
	\end{subfigure}
    \begin{subfigure}[b]{0.36\textheight}
		\includegraphics[trim={1.5cm 0.5cm 1cm 1cm},clip,width=\linewidth]{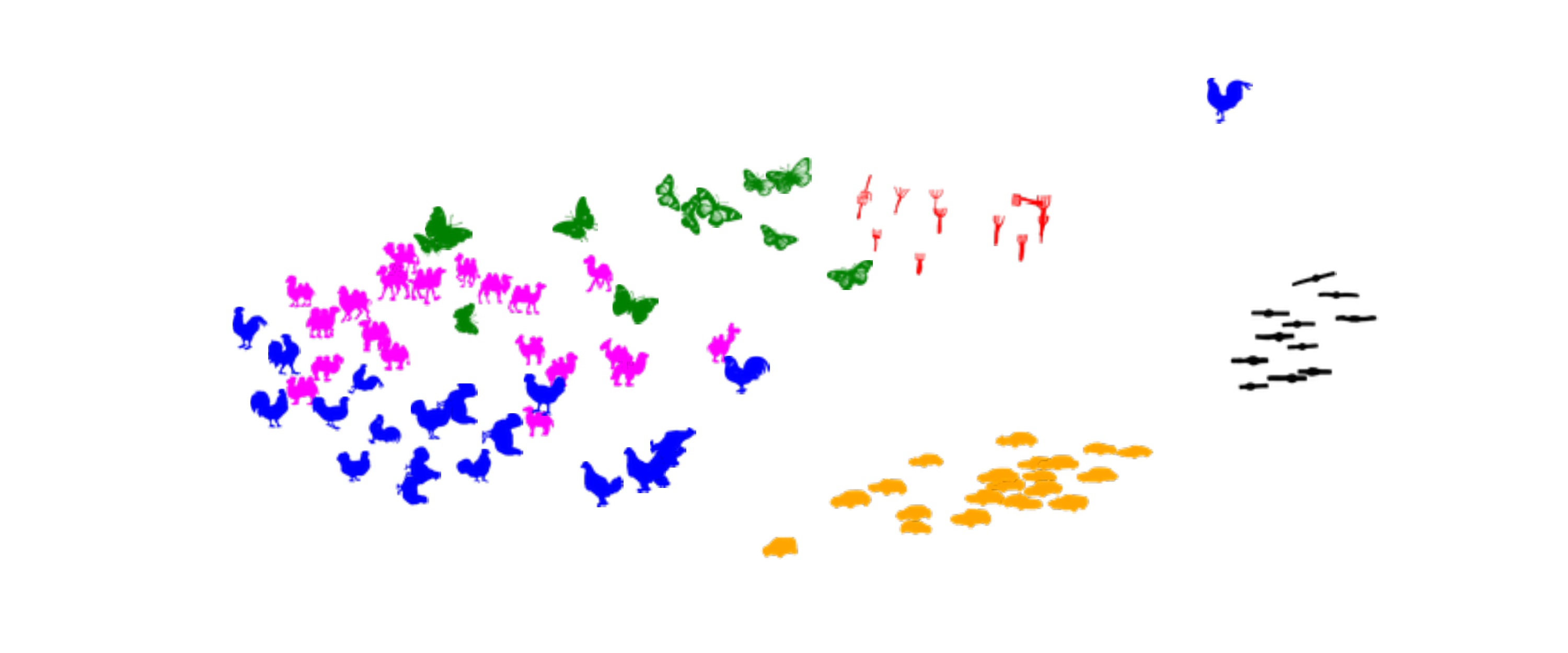}
		\caption*{TMDS}
	\end{subfigure}

	\caption{Embedding of six classes from MPEG7 dataset. It can be seen that TMDS fixes the embedding (note the green class) but yet it is still imperfect (note the blue rooster)}
    \label{fig:mpeg7_imperfect}
\end{figure*}

\begin{figure*}[h]
	\centering
	\begin{subfigure}[b]{0.24\textheight}
		\includegraphics[trim={1.5cm 0cm 1cm 1cm },clip,width=\linewidth]{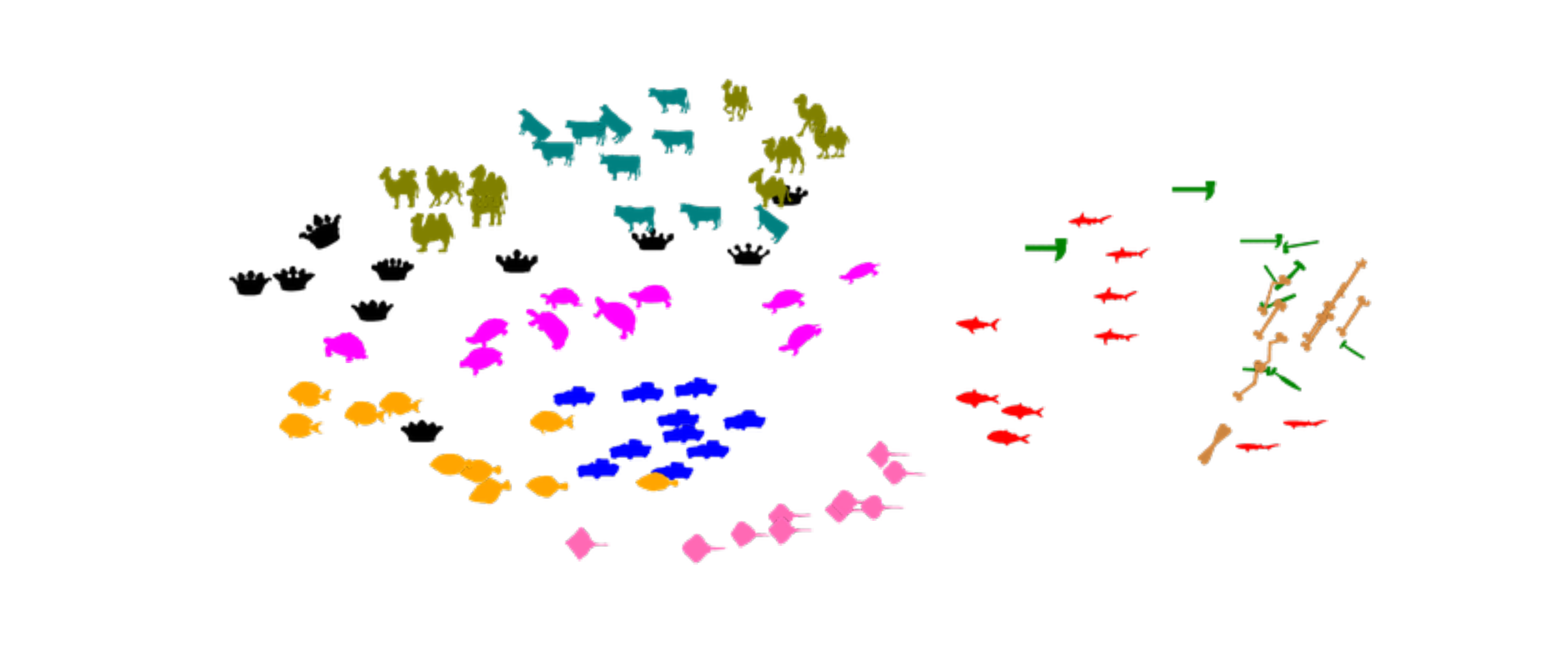}
		\caption*{SMACOF}
	\end{subfigure}
    \begin{subfigure}[b]{0.24\textheight}
		\includegraphics[trim={1.5cm 0cm 1cm 1cm },clip,width=\linewidth]{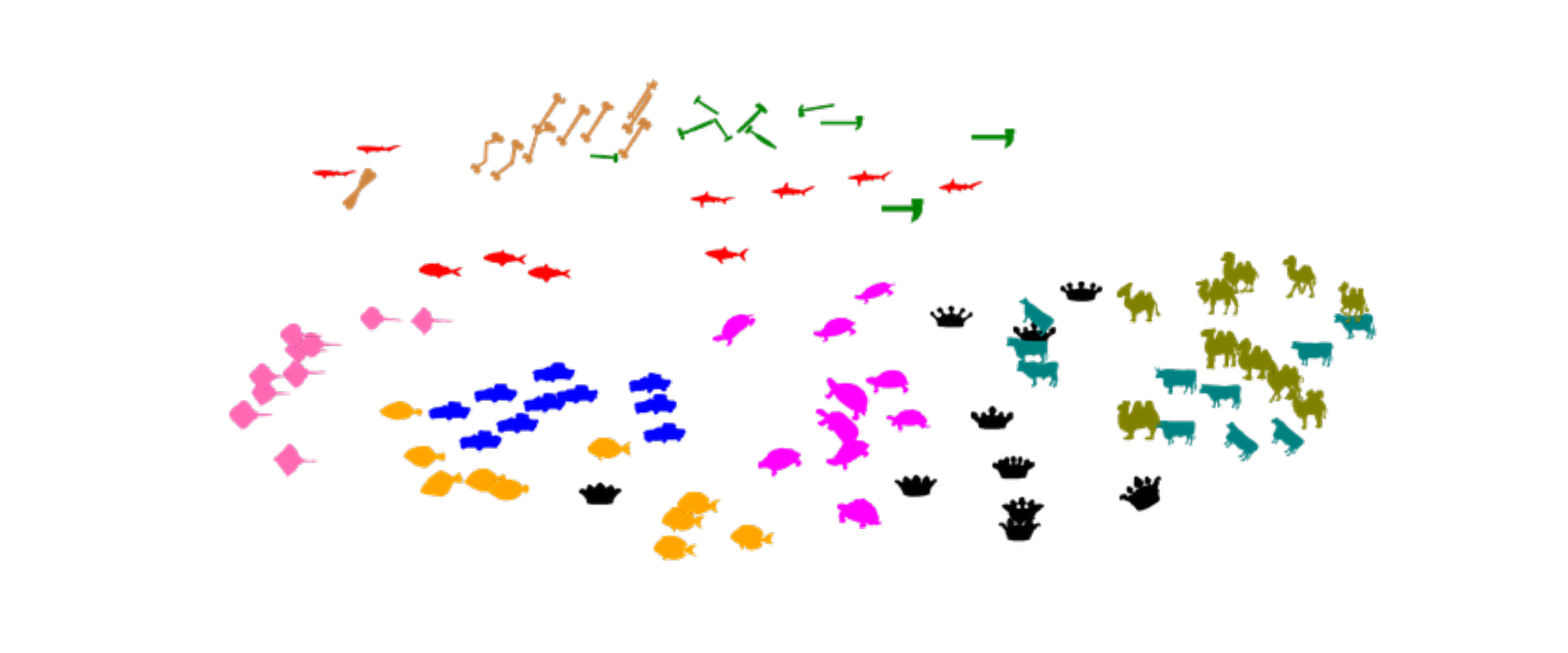}
		\caption*{FG12}
	\end{subfigure}
    \begin{subfigure}[b]{0.24\textheight}
		\includegraphics[trim={1.5cm 0cm 1cm 1cm },clip,width=\linewidth]{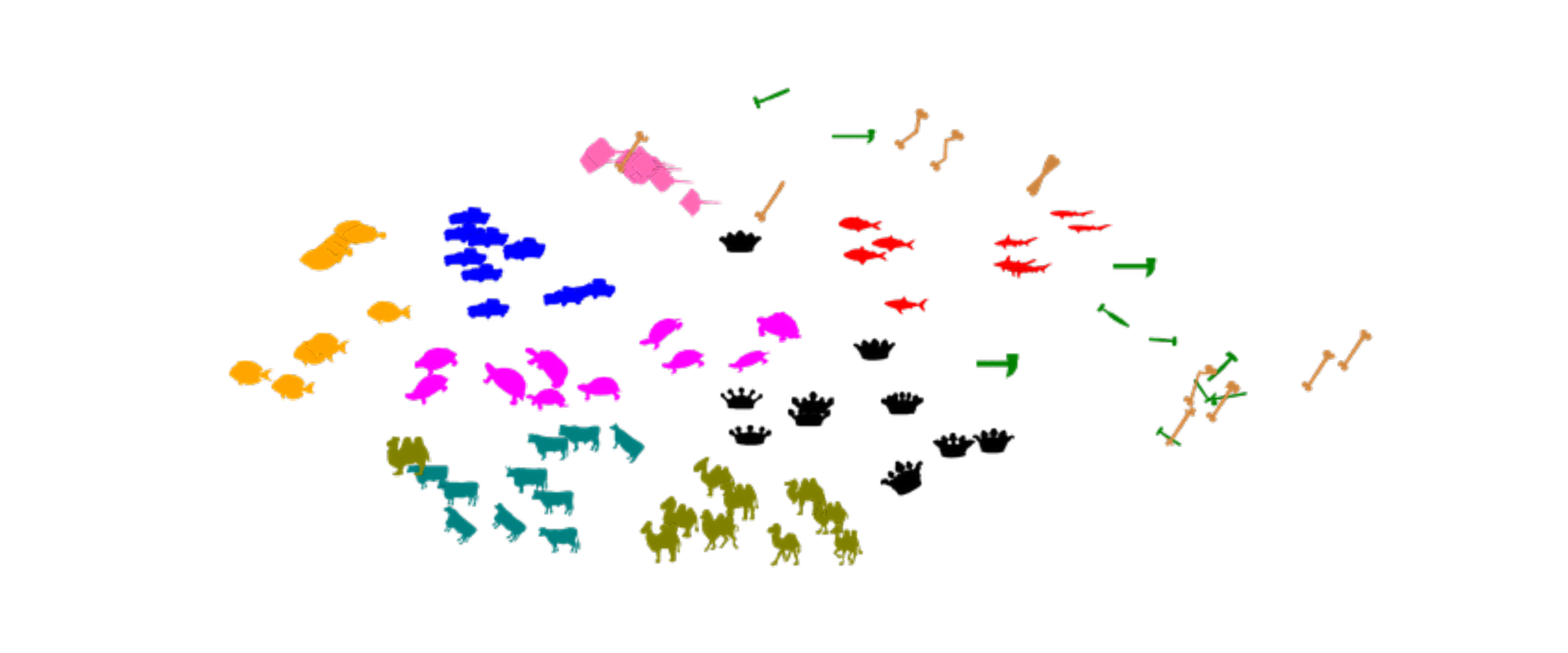}
		\caption*{TMDS}
	\end{subfigure}
	
	\caption{Embedding of 100 random shapes, selected from 10 classes (10 shapes per class) from the 1070db dataset. TMDS improves the embedding compared to SMACOF (note the olive color) and FG12 (note the red and the yellow classes). The silhouette scores improve from 0.14 (SMACOF) and 0.19 (FG12), to 0.27 (TMDS)}
	\label{fig:1070db_10class}
\end{figure*}

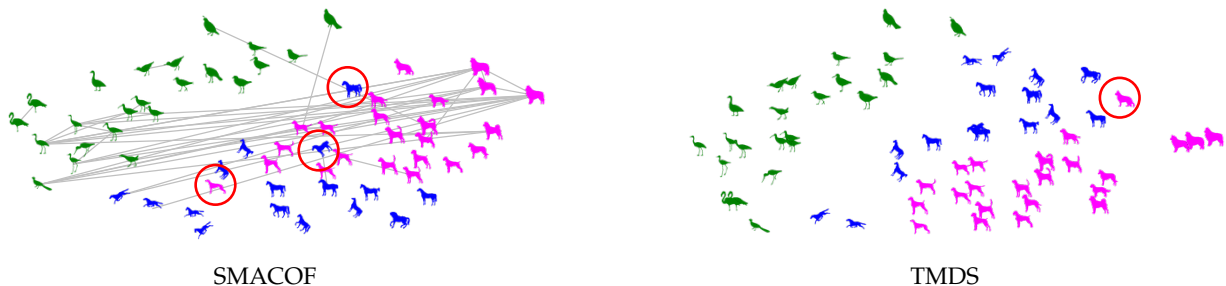
\begin{figure*}[h]
	\centering
    \begin{subfigure}[b]{0.36\textheight}
    	\begin{tikzpicture}
		\node(a) at (0,0,0){\includegraphics[trim={1.5cm 1.5cm 1cm 1cm},clip,width=\linewidth]{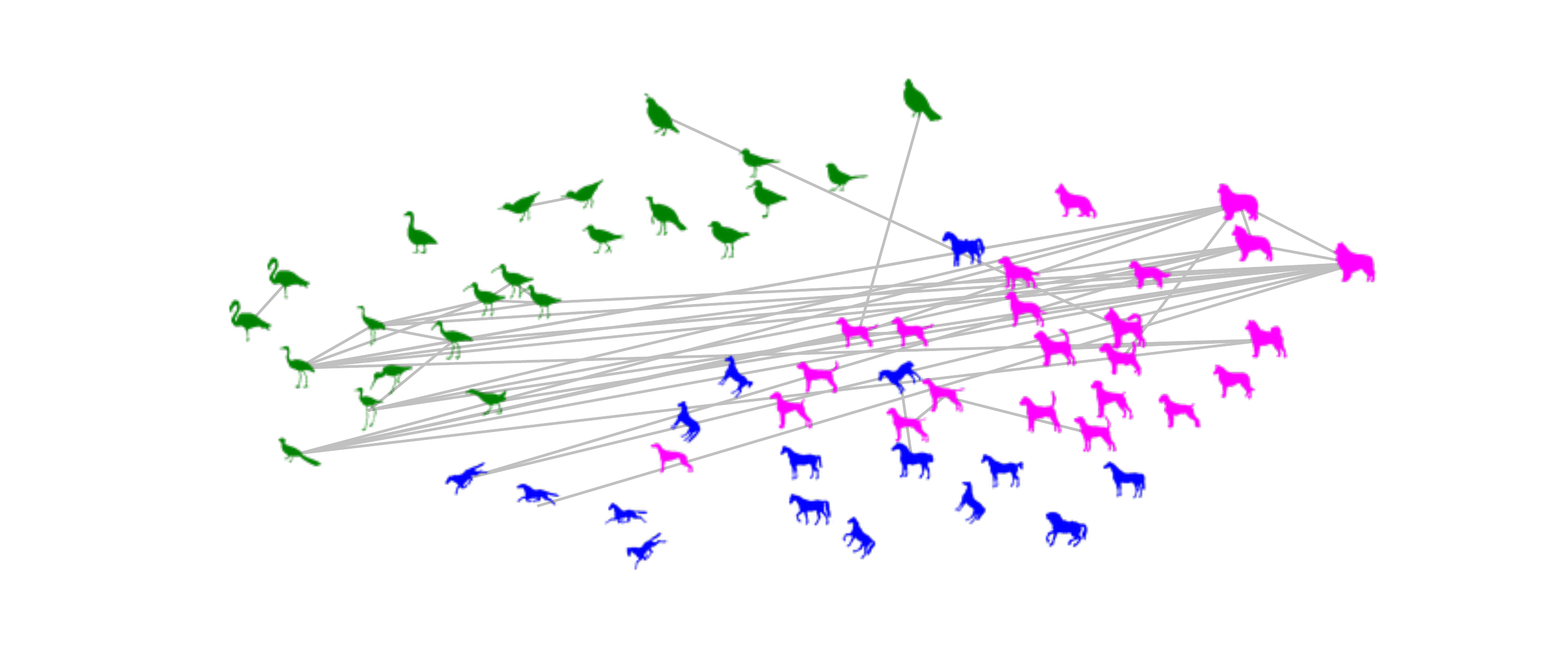}};
        \node at(a.center)[draw, red,line width=1pt,ellipse, minimum width=15pt, minimum height=15pt,xshift=28pt, yshift=12pt]{};
        \node at(a.center)[draw, red,line width=1pt,ellipse, minimum width=15pt, minimum height=15pt,xshift=17pt, yshift=-12pt]{};
        \node at(a.center)[draw, red,line width=1pt,ellipse, minimum width=15pt, minimum height=15pt,xshift=-22pt, yshift=-25pt]{};
        \end{tikzpicture}
		\caption*{SMACOF}
	\end{subfigure}
    \begin{subfigure}[b]{0.36\textheight}
    	\begin{tikzpicture}
		\node(a) at (0,0,0){\includegraphics[trim={1.5cm 1.5cm 1cm 1cm},clip,width=\linewidth]{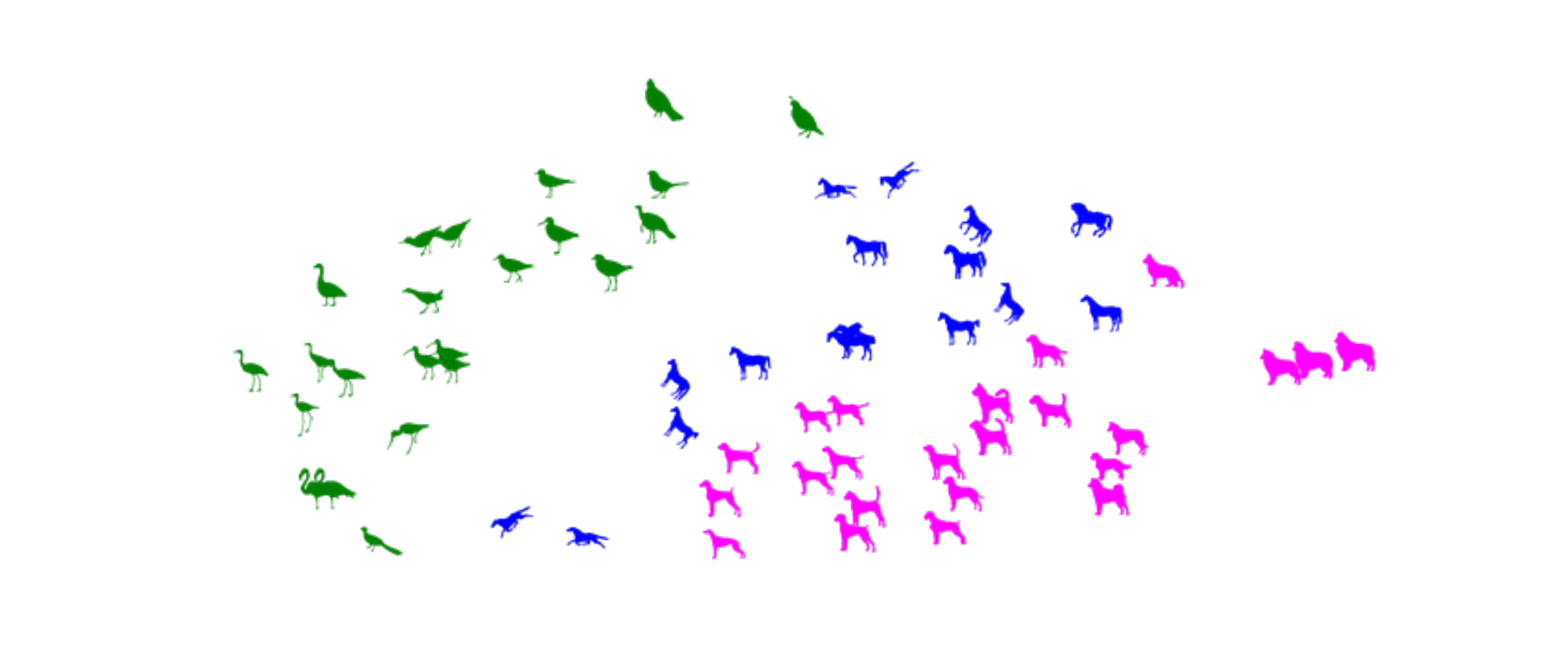}};
        \node at(a.center)[draw, red,line width=1pt,ellipse, minimum width=15pt, minimum height=15pt,xshift=63pt, yshift=8pt]{};
        \end{tikzpicture}
		\caption*{TMDS}
	\end{subfigure}
    \caption{Embedding of three classes from the 1070db dataset. The gray lines present the outliers that were filtered by TMDS (44 out of 1847; that is $2.3\%$ ). TMDS improves the embedding of the magenta and blue classes. The red circles illustrate shapes that are clearly not located close enough to the classes. Note that TMDS does not guarantee a perfect embedding.}
    \label{fig:mpeg7_with_lines}
\end{figure*}

\begin{table*}[h]
\centering
\begin{tabular}{lllllll}
\hline
& \multicolumn{2}{l}{Motion Data} & \multicolumn{2}{l}{MPEG7} & \multicolumn{2}{l}{3D Mesh} \\
\hline
Scoring &  SMACOF    & TMDS &  SMACOF    & TMDS &  SMACOF    & TMDS\\
\hline
Silhouette        & 0.282    & 0.321 & 0.38  & 0.48 & 0.39 & 0.49  \\
Calinski Harabaz  & 167.2    & 218.9 & 60.54 & 90.5 & 50.7 & 59.0 \\
AMI               & 0.64     & 0.66 & 0.78 & 0.85 & 0.72 & 0.79\\
Completeness      & 0.65     & 0.68 & 0.79 & 0.86 & 0.80 & 0.83\\
Homogeneity       & 0.64     & 0.67 & 0.82 & 0.87 & 0.76  & 0.82\\
NMI               & 0.65     & 0.674 & 0.81 & 0.86 & 0.78 & 0.82\\
\hline
\end{tabular}
\caption{TMDS outperforms SAMCOF in the three datasets, using six common measures. }
\label{table:andreas_score}
\end{table*}
\section{Distribution of distances}

\begin{figure*}[t]
	\centering
	\begin{subfigure}{0.2\textheight}
		\includegraphics[trim={1cm 0 1cm 0},clip,width=\linewidth]{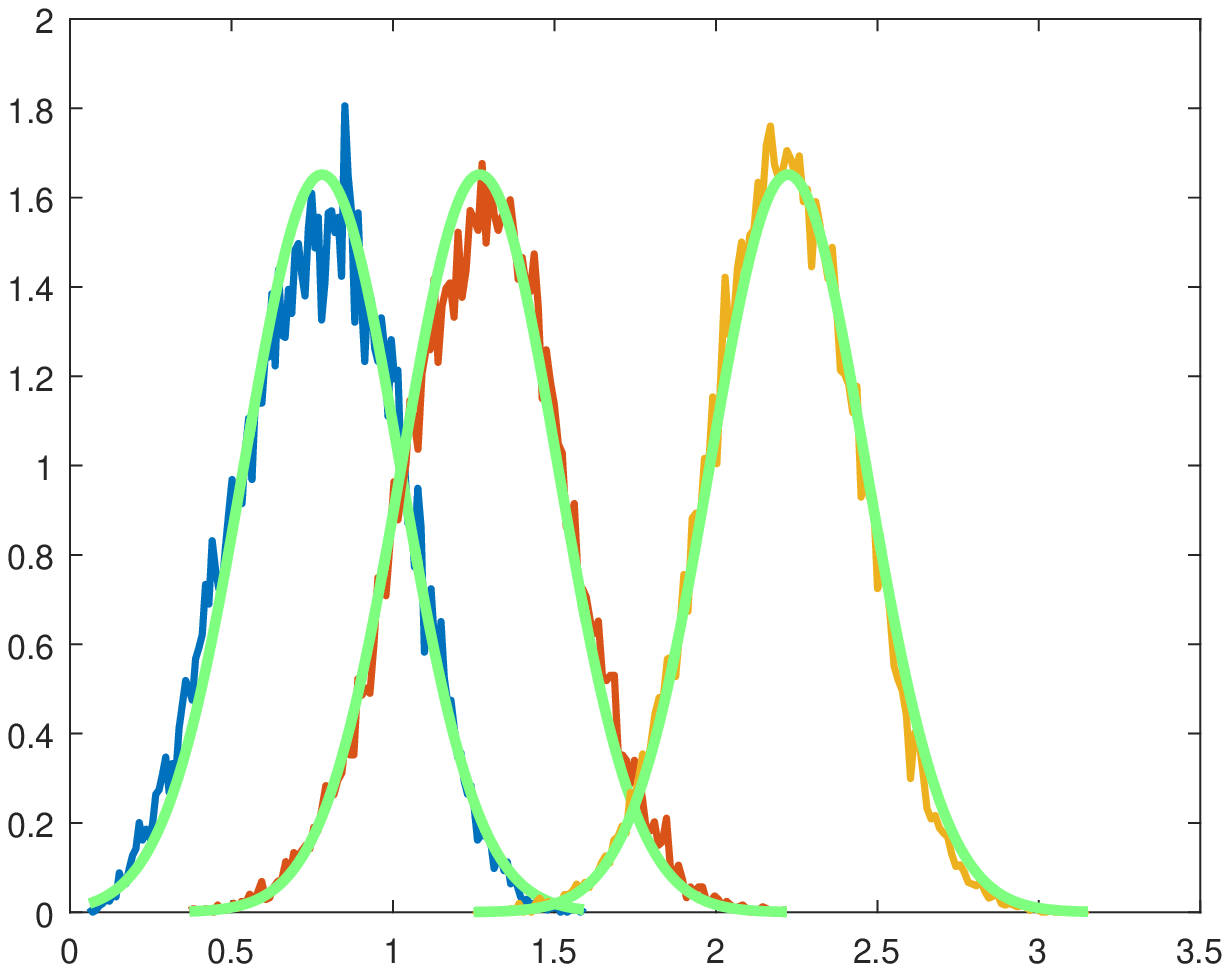}
		\caption{}
	\end{subfigure}
	\begin{subfigure}{0.2\textheight}
		\includegraphics[trim={1cm 0 1cm 0},clip,width=\linewidth]{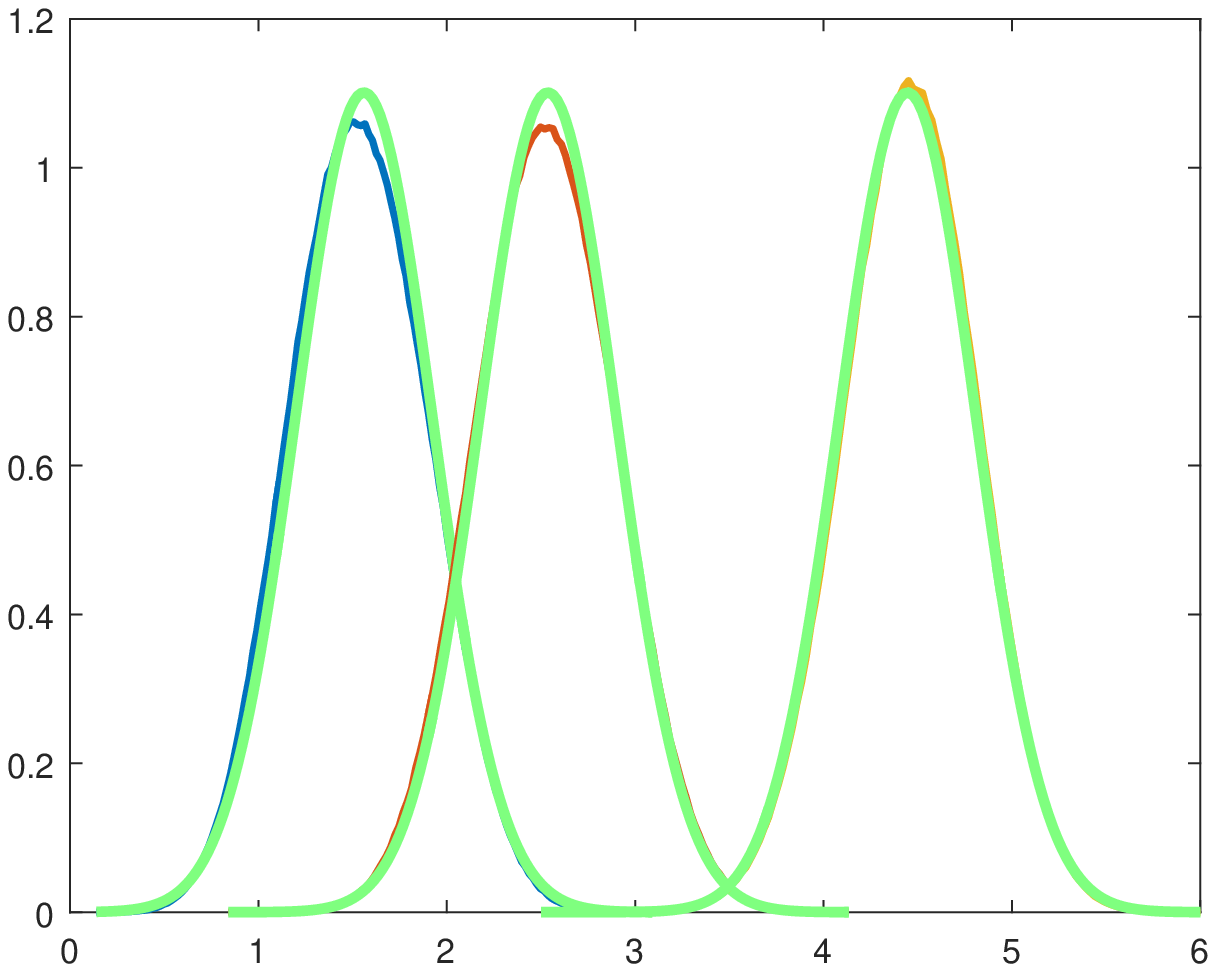}
		\caption{}
	\end{subfigure}
	\begin{subfigure}{0.2\textheight}
		\includegraphics[trim={1cm 0 1cm 0},clip,width=\linewidth]{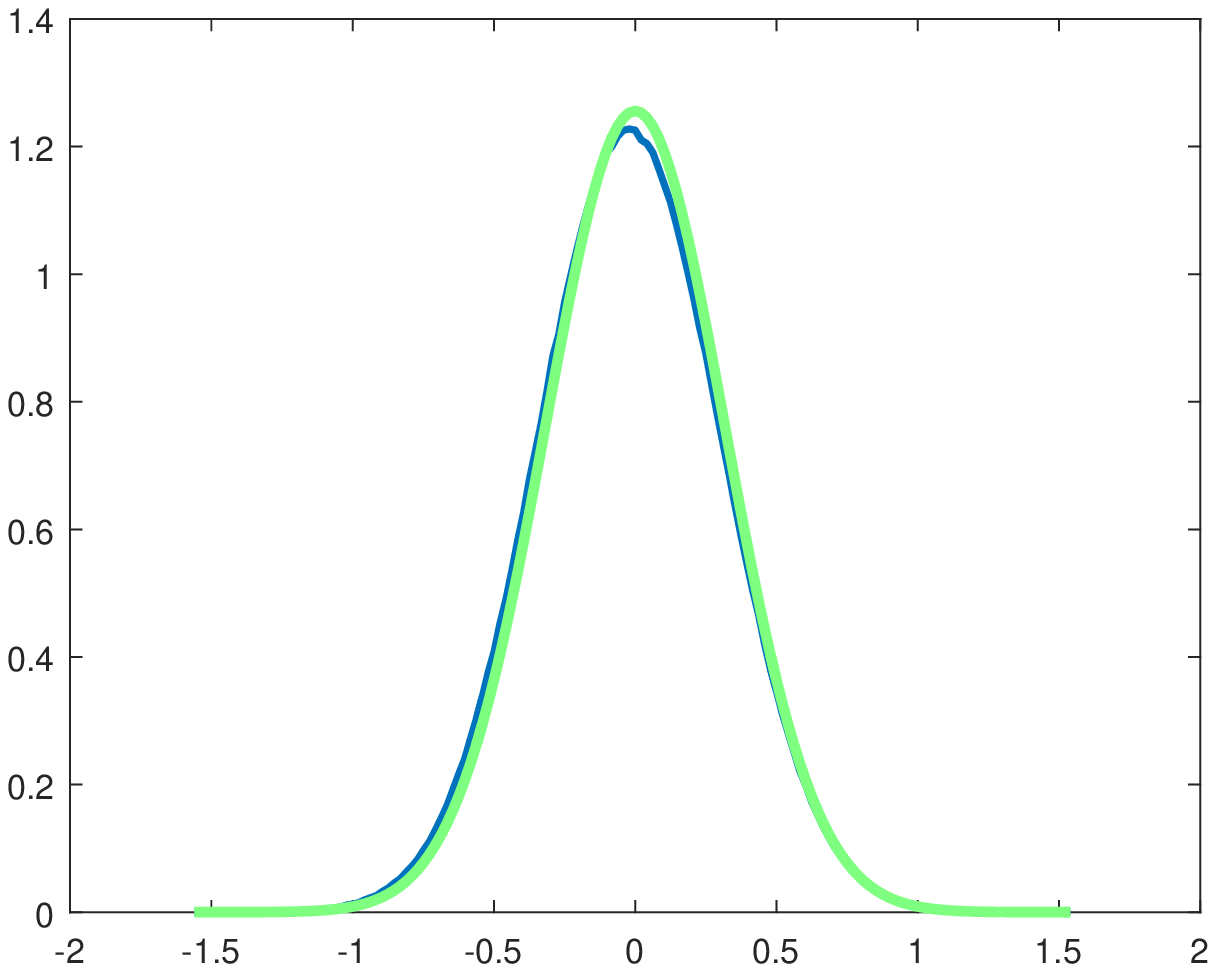}
		\caption{}
	\end{subfigure}

	\caption{The approximate normal distributions of $D(p_1,p_2)$,  $PD$, and $MD$ are displayed in (a), (b) and (c), respectively.
	The green curve stands for the approximated values, while the blue, red and orange plots represent the measured distance distributions for dimensions $6,10$, and $30$, respectively. In (c) all the curves are alike, therefore we display only the one obtained for 6D.
  }
	\label{fig:distributions}
\end{figure*}

In this section, we seek to analyze the probability that a triangle is broken by an erroneous edge, 
assuming that errors have the same distribution as distances. 
Thus, given a triangle formed by $p_1,p_2,p_3$ in a $d$-dimensional unit hypercube, we should learn the distribution of distances of triangle edges to facilitate the estimation of the probability that an erroneous edge breaks the triangle inequality. 

We begin by analyzing the distribution of edge lengths.
Let $ p_1, p_2 \sim U([0, 1]^d) $ be two points in a $d$-dimensional unit hypercube, uniformly randomly sampled.

The distance $D(p_1, p_2) = \sqrt{\sum_{i=1}^d { \left (p_1^{(i)} - p_2^{(i)} \right )}^2}$ has normal distribution
with mean and variance of  $ \sqrt{\frac{d}{6}} $ and $ \frac{7}{120} $, respectively, and the approximation gets better as $d$ increases. 
The proof is provided in \cite{SE_1985698}.

Hereafter, we denote $\mu = \sqrt{\frac{d}{6}}$ and $\sigma^2 = \frac{7}{120}$.

We empirically calculated $Cov\left( D(p_1, p_2),D(p_2, p_3)\right)$ - the covariance between $D(p_1, p_2)$ and $D(p_2, p_3)$, for several dimensions, and found that it has a constant value of $0.008$. This allows us to approximate the mean $E$ and variance $Var$ for $PD = D(p_1, p_2) + D(p_2, p_3) $ and $MD = D(p_1, p_2) - D(p_2, p_3)$, assuming that their distributions are normal:

\begin{fleqn}
	\begin{align*}
	E(PD) &= 2\mu \\
	Var(PD) &= 2\sigma^2 + 2Cov(D(p_1, p_2),D(p_2, p_3))  \\
	&=2\sigma^2 + 0.016 \approxeq 2.27\sigma^2 \\
	E(MD) &= \mu - \mu = 0 \\
	Var(MD) &= 2\sigma^2 - 2Cov(D(p_1, p_2),D(p_2, p_3)) \\
	&=  2\sigma^2 - 0.016 \approxeq  1.73\sigma^2
	\end{align*}
\end{fleqn}

These normal distributions are displayed in Figure \ref{fig:distributions} for dimensions $6,10$, and $30$.
As can be noted, the accuracy of the approximations increases for higher dimensions.

Given the above distributions, we are now ready to estimate the probability of a triangle to break by an erroneous edge.

\begin{theorem}
	Let $p_1,p_2,p_3$ be three points uniformly randomly and independently sampled in a  $d$-dimensional unit hypercube. 
Given the above distance distributions, and let the edge $(p_1, p_3)$ be an outlier distance, the probability that the triangle $p_1,p_2,p_3$ is broken is 	
	$$ 2 \times \Phi(\frac{-\mu}{\sqrt{2.73}\sigma}) + \Phi(\frac{-\mu}{\sqrt{3.27}\sigma}),$$ where $\Phi$ stands for standard normal cumulative distribution function.
\end{theorem}	

\begin{proof}
	It was shown above that
	$$D(p_1,p_3) \sim Norm(\mu, \sigma^2 )$$
	$$ D(p_1,p_2) + D(p_2,p_3) \sim Norm(2\mu, 2.27\sigma^2 )$$
	$$ D(p_2,p_3) - D(p_1,p_2) \sim Norm(0, 1.73\sigma^2 ),$$
    $$ D(p_1,p_2) - D(p_2,p_3) \sim Norm(0, 1.73\sigma^2 ) $$
	where $Norm$ stands for normal distribution.
	
	Assuming that $D(p_1,p_3)$ and $D(p_1, p_2)$, as well as  $D(p_1,p_3)$ and $D(p_2, p_3)$ are independent, a triangle breaks if one of the following inequalities is satisfied:
	$$ D(p_1,p_2) + D(p_2,p_3) < D(p_1,p_3)  $$
	$$ D(p_1,p_2) + D(p_1,p_3) < D(p_2,p_3)  $$
	$$ D(p_1,p_3) + D(p_2,p_3) < D(p_1,p_2)  $$
	
	Then, the probabilities for these independent cases are:
	$$ Pr(D(p_1,p_2) + D(p_2,p_3) < D(p_1,p_3)) = 
	\Phi(\frac{-\mu}{\sqrt{3.27}\sigma})  $$
	$$ Pr(D(p_1,p_2) + D(p_1,p_3) < D(p_2,p_3)) = 
	\Phi(\frac{-\mu}{\sqrt{2.73}\sigma}) $$
	$$ Pr(D(p_1,p_3) + D(p_2,p_3) < D(p_1,p_2)) =
	\Phi(\frac{-\mu}{\sqrt{2.73}\sigma}) $$
	
	Thus, the probability for a triangle to break due to an outlier edge is simply their sum:
	$$ 2 \times \Phi(\frac{-\mu}{\sqrt{2.73}\sigma}) + \Phi(\frac{-\mu}{\sqrt{3.27}\sigma})$$
\end{proof}

Note that the assumption that an outlier behaves similarly to a common distance (i.e., same distributions) makes its detection harder.Nevertheless, as we showed above, an outlier breaks a large number of triangles. For example, assuming that triangles are independent, with $d=2$ , $24\%$ of the triangles that are associated with an outlier edge are broken. This number is validated by our empirical evaluations above. Also, note that $\mu$ is a function of $d$, thus as the dimension $d$ increases, $\mu$ also increases, and the probability above decreases.

\section{Conclusions}

We presented a technique to filter distances prior to applying MDS, so that it is more robust to outliers. The technique analyzes the triangles formed by three points and their pairwise distances, and associates distances to broken triangles. Not every outlier is associated with a broken triangle, and an edge associated with a broken triangle is not necessarily an outlier. 
Thus, it is expected to produce both false positives and false negatives. However, as we showed, as long as the portion of outlier edges is reasonable, i.e., $20\%$, the false positives are non-destructive, and the quality of the embedding is high. Most notably, when the number of outliers is particularly small, 
as in reasonable real-world scenarios, the accuracy of our technique is particularly high and the improvement over a direct MDS is significant. We also showed that our technique is useful to distill the data before applying MDS, even when there are no significant outliers. 

In our work, we focused on Euclidean metric, however, the technique can also be applied to other metrics. For example, in psychology and marketing it is common to use the family of Minkowsky distances \cite{zhang2007market,jaworska2009review}:
$$ \left(\sum_{i=1}^n |x_i-y_i|^p\right)^{1/p}.$$
As long as $ p > 1$, our method applies.

Beyond the generalization to other metrics, our technique is applicable to general embedding, and not restricted to a specific stress function or MDS algorithm. As long as the embedding method does not require a full dissimilarity matrix, our method is viable. 

We would like to stress that our method has no parameters. The threshold $\phi$ uses the value of $|E| / 2$, which is parameter-free. However, one can define $\phi$ with a parameter that reflects the expected number of outliers, to refine the accuracy of the method, i.e., to produce less false-positives.

\bibliographystyle{IEEEtran}
\bibliography{mds} 

\begin{thebibliography}{10}
\providecommand{\url}[1]{#1}
\csname url@samestyle\endcsname
\providecommand{\newblock}{\relax}
\providecommand{\bibinfo}[2]{#2}
\providecommand{\BIBentrySTDinterwordspacing}{\spaceskip=0pt\relax}
\providecommand{\BIBentryALTinterwordstretchfactor}{4}
\providecommand{\BIBentryALTinterwordspacing}{\spaceskip=\fontdimen2\font plus
\BIBentryALTinterwordstretchfactor\fontdimen3\font minus
  \fontdimen4\font\relax}
\providecommand{\BIBforeignlanguage}[2]{{%
\expandafter\ifx\csname l@#1\endcsname\relax
\typeout{** WARNING: IEEEtran.bst: No hyphenation pattern has been}%
\typeout{** loaded for the language `#1'. Using the pattern for}%
\typeout{** the default language instead.}%
\else
\language=\csname l@#1\endcsname
\fi
#2}}
\providecommand{\BIBdecl}{\relax}
\BIBdecl

\bibitem{de1988convergence}
J.~De~Leeuw, ``Convergence of the majorization method for multidimensional
  scaling,'' \emph{Journal of classification}, vol.~5, no.~2, pp. 163--180,
  1988.

\bibitem{leeuw2008multidimensional}
J.~d. Leeuw and P.~Mair, ``Multidimensional scaling using majorization: Smacof
  in r,'' 2008.

\bibitem{de2004sparse}
V.~De~Silva and J.~B. Tenenbaum, ``Sparse multidimensional scaling using
  landmark points,'' Technical report, Stanford University, Tech. Rep., 2004.

\bibitem{chan2009efficient}
F.~K. Chan and H.-C. So, ``Efficient weighted multidimensional scaling for
  wireless sensor network localization,'' \emph{IEEE Transactions on Signal
  Processing}, vol.~57, no.~11, pp. 4548--4553, 2009.

\bibitem{spence1989robust}
I.~Spence and S.~Lewandowsky, ``Robust multidimensional scaling,''
  \emph{Psychometrika}, vol.~54, no.~3, pp. 501--513, 1989.

\bibitem{forero2012sparsity}
P.~A. Forero and G.~B. Giannakis, ``Sparsity-exploiting robust multidimensional
  scaling,'' \emph{IEEE Transactions on Signal Processing}, vol.~60, no.~8, pp.
  4118--4134, 2012.

\bibitem{cayton2006robust}
L.~Cayton and S.~Dasgupta, ``Robust euclidean embedding,'' in \emph{Proceedings
  of the 23rd international conference on machine learning}.\hskip 1em plus
  0.5em minus 0.4em\relax ACM, 2006, pp. 169--176.

\bibitem{Krusk64_0}
J.~B. Kruskal, ``Nonmetric multidimensional scaling: a numerical method,'' pp.
  115--129, 1964.

\bibitem{Shep62_0}
R.~N. Shepard, ``The analysis of proximities: multidimensional scaling with an
  unknown distance function. i.'' pp. 125--140, 1962.

\bibitem{Market_0}
L.~A. Neidell, ``The use of nonmetric multidimensional scaling in marketing
  analysis,'' pp. 37--43, 1969.

\bibitem{shepard1980multidimensional}
R.~N. Shepard, ``Multidimensional scaling, tree-fitting, and clustering,''
  \emph{Science}, vol. 210, no. 4468, pp. 390--398, 1980.

\bibitem{borg2005modern}
I.~Borg and P.~J. Groenen, \emph{Modern multidimensional scaling: Theory and
  applications}.\hskip 1em plus 0.5em minus 0.4em\relax Springer Science \&
  Business Media, 2005.

\bibitem{buja2008data}
A.~Buja, D.~F. Swayne, M.~L. Littman, N.~Dean, H.~Hofmann, and L.~Chen, ``Data
  visualization with multidimensional scaling,'' \emph{Journal of Computational
  and Graphical Statistics}, vol.~17, no.~2, pp. 444--472, 2008.

\bibitem{buja2002visualization}
A.~Buja and D.~F. Swayne, ``Visualization methodology for multidimensional
  scaling,'' \emph{Journal of Classification}, vol.~19, no.~1, pp. 7--43, 2002.

\bibitem{zigelman2002texture}
G.~Zigelman, R.~Kimmel, and N.~Kiryati, ``Texture mapping using surface
  flattening via multidimensional scaling,'' \emph{IEEE Transactions on
  Visualization and Computer Graphics}, vol.~8, no.~2, pp. 198--207, 2002.

\bibitem{chen20103d}
Z.~Chen and K.~Tang, ``3d shape classification based on spectral function and
  mds mapping,'' \emph{Journal of Computing and Information Science in
  Engineering}, vol.~10, no.~1, p. 011004, 2010.

\bibitem{pickup2014shrec}
D.~Pickup, X.~Sun, P.~L. Rosin, R.~Martin, Z.~Cheng, Z.~Lian, M.~Aono, A.~B.
  Hamza, A.~Bronstein, M.~Bronstein \emph{et~al.}, ``Shrec’14 track: Shape
  retrieval of non-rigid 3d human models,'' \emph{Proc. 3DOR}, vol.~4, no.~7,
  p.~8, 2014.

\bibitem{Li:2015}
\BIBentryALTinterwordspacing
Y.~Li, H.~Su, C.~R. Qi, N.~Fish, D.~Cohen-Or, and L.~J. Guibas, ``Joint
  embeddings of shapes and images via cnn image purification,'' \emph{ACM
  Trans. Graph.}, vol.~34, no.~6, pp. 234:1--234:12, Oct. 2015. [Online].
  Available: \url{http://doi.acm.org/10.1145/2816795.2818071}
\BIBentrySTDinterwordspacing

\bibitem{sammon1969nonlinear}
J.~W. Sammon, ``A nonlinear mapping for data structure analysis,'' \emph{IEEE
  Transactions on computers}, vol.~18, no.~5, pp. 401--409, 1969.

\bibitem{CitiesDataset}
J.~Burkardt, \emph{CITIES - City Distance Datasets}, 2011.

\bibitem{denoeux2004evclus}
T.~Den{\oe}ux and M.-H. Masson, ``Evclus: evidential clustering of proximity
  data,'' \emph{IEEE Transactions on Systems, Man, and Cybernetics, Part B
  (Cybernetics)}, vol.~34, no.~1, pp. 95--109, 2004.

\bibitem{graepel1999classification}
T.~Graepel, R.~Herbrich, P.~Bollmann-Sdorra, and K.~Obermayer, ``Classification
  on pairwise proximity data,'' \emph{Advances in neural information processing
  systems}, pp. 438--444, 1999.

\bibitem{ling2005using}
H.~Ling and D.~W. Jacobs, ``Using the inner-distance for classification of
  articulated shapes,'' in \emph{Computer Vision and Pattern Recognition, 2005.
  CVPR 2005. IEEE Computer Society Conference on}, vol.~2.\hskip 1em plus 0.5em
  minus 0.4em\relax IEEE, 2005, pp. 719--726.

\bibitem{1070shapes}
Brown-University, \emph{1070 Binary Shape Databases},
  \url{http://vision.lems.brown.edu/content/available-software-and-databases},,
  2005.

\bibitem{Chen2009ABF}
X.~Chen, A.~Golovinskiy, and T.~Funkhouser, ``A benchmark for {3D} mesh
  segmentation,'' \emph{ACM Transactions on Graphics}, vol.~28, no.~3, 2009.

\bibitem{Aristidou:2015_CGF}
A.~Aristidou, P.~Charalambous, and Y.~Chrysanthou, ``Emotion analysis and
  classification: Understanding the performers emotions using the {LMA}
  entities,'' \emph{Computer Graphics Forum}, vol.~34, no.~6, p. 262–276,
  September 2015.

\bibitem{kleiman2015shed}
Y.~Kleiman, O.~van Kaick, O.~Sorkine-Hornung, and D.~Cohen-Or, ``Shed: shape
  edit distance for fine-grained shape similarity,'' \emph{ACM Transactions on
  Graphics (TOG)}, vol.~34, no.~6, p. 235, 2015.

\bibitem{rousseeuw1987silhouettes}
P.~J. Rousseeuw, ``Silhouettes: a graphical aid to the interpretation and
  validation of cluster analysis,'' \emph{Journal of computational and applied
  mathematics}, vol.~20, pp. 53--65, 1987.

\bibitem{calinski1974dendrite}
T.~Cali{\'n}ski and J.~Harabasz, ``A dendrite method for cluster analysis,''
  \emph{Communications in Statistics-theory and Methods}, vol.~3, no.~1, pp.
  1--27, 1974.

\bibitem{vinh2010information}
N.~X. Vinh, J.~Epps, and J.~Bailey, ``Information theoretic measures for
  clusterings comparison: Variants, properties, normalization and correction
  for chance,'' \emph{Journal of Machine Learning Research}, vol.~11, no. Oct,
  pp. 2837--2854, 2010.

\bibitem{strehl2002cluster}
A.~Strehl and J.~Ghosh, ``Cluster ensembles---a knowledge reuse framework for
  combining multiple partitions,'' \emph{Journal of machine learning research},
  vol.~3, no. Dec, pp. 583--617, 2002.

\bibitem{rosenberg2007v}
A.~Rosenberg and J.~Hirschberg, ``V-measure: A conditional entropy-based
  external cluster evaluation measure.'' in \emph{EMNLP-CoNLL}, vol.~7, 2007,
  pp. 410--420.

\bibitem{SE_1985698}
Henry, \emph{How is the distance of two random points in a unit hypercube
  distributed?}, 2016.

\bibitem{zhang2007market}
Y.~Zhang, J.~Jiao, and Y.~Ma, ``Market segmentation for product family
  positioning based on fuzzy clustering,'' \emph{Journal of Engineering
  Design}, vol.~18, no.~3, pp. 227--241, 2007.

\bibitem{jaworska2009review}
N.~Jaworska and A.~Chupetlovska-Anastasova, ``A review of multidimensional
  scaling (mds) and its utility in various psychological domains,''
  \emph{Tutorials in Quantitative Methods for Psychology}, vol.~5, no.~1, pp.
  1--10, 2009.

\end{thebibliography}

\end{document}